\documentclass{IEEEtaes}
\usepackage{amsmath,amsfonts}
\usepackage{algorithmic}
\usepackage{algorithm}
\usepackage{array}
\usepackage[caption=false,font=normalsize,labelfont=sf,textfont=sf]{subfig}
\usepackage{textcomp}
\usepackage{stfloats}
\usepackage{url}
\usepackage{verbatim}
\usepackage{graphicx}
\usepackage{cite}
\usepackage{multirow}
\usepackage[caption=false,font=normalsize,labelfont=sf,textfont=sf]{subfig}
\usepackage{color}
\usepackage{xcolor}
\usepackage{amsthm,amssymb,bm}

\usepackage{colortbl} 

\usepackage{mathtools}
\usepackage{array}       
\usepackage{booktabs}    
\usepackage{caption}
\usepackage{tabularx}

\newtheorem{lemma}{Lemma}

\newtheorem{assumption}{Assumption}
\newtheorem{theorem}{Theorem} 

\jvol{XX}
\jnum{XX}
\jmonth{XXXXX}
\paper{1234567}
\pubyear{2022}
\doiinfo{TAES.2022.Doi Number}

\begin{document}

\title{Cooperative Circumnavigation for Multiple Unmanned Surface Vehicles Without External Localization}

\author{Xueming Liu$^{\dagger}$}
\author{Lin Li$^{\dagger}$}
\author{Xiang Zhou}
\author{Tianjiang Hu}
\author{Qingrui Zhang}
\affil{School of Aeronautics and Astronautics, Sun Yat-sen University (Shenzhen Campus), Shenzhen, China}

\markboth{AUTHOR ET AL.}{SHORT ARTICLE TITLE}

\maketitle

\begin{abstract}
This paper proposes a cooperative target circumnavigation framework for multiple unmanned surface vehicles (USVs) operating without external localization. The objective is to maintain a uniform circular formation of a specified radius around a target using only limited onboard sensing. The framework adopts a heterogeneous perception strategy that distinguishes between the asymmetric sensing relationships with the target and among the USVs. Specifically, the USVs obtain relative range and displacement measurements through active perception and inter-vehicle communication, while bearing measurements to a non-cooperative target are acquired via passive sensors. To estimate relative positions — both among USVs and between each USV and the target — we employ a Maximum Correntropy Kalman Filter and a Pseudo-Linear Kalman Filter, respectively. A coupled-oscillator-based formation controller is designed to ensure system observability while achieving circumnavigation. Theoretical analysis demonstrates that the controller ensures the relative motions between the USVs, as well as that between each USV and the target, satisfy the persistent excitation condition, thereby guaranteeing observability of the Kalman-based filters. The effectiveness of the proposed approach is validated through numerical simulations.
\end{abstract}


\renewcommand{\arraystretch}{ 1.1}
\begin{table}[htbp]
\centering
\caption*{KEY NOMENCLATURE}
\begin{tabular}{@{}l p{5.75cm}@{}}
$(\cdot)_{i}, (\cdot)_{j}, (\cdot)_{o}$ & State of USV $i$, USV $j$, or target $o$. \\
$(\cdot)_{ij}, (\cdot)_{io}$ & Relative states or parameters for inter-USV or USV-Target pairs. \\
$\boldsymbol{p},\boldsymbol{v}\in \mathbb{R}^2,\boldsymbol{x}\in \mathbb{R}^4$ & Position, velocity, and their combined state vector. \\
$\theta_{i}, v_{i}, \omega_{i}\in \mathbb{R}$ & Yaw angle, surge speed, and yaw rate of USV $i$. \\
$\boldsymbol{\delta},\boldsymbol{b}\in \mathbb{R}^2$, ${r}\in \mathbb{R}$ & Displacement, bearing vector, and range. \\
$\boldsymbol{z}_{io}\in \mathbb{R}^2, z_{ij}\in \mathbb{R}$ & Equivalent measurements for USV-Target and inter-USV cases, respectively.\\
$\hat{(\cdot)}^{-}, \hat{(\cdot)}^{+}$ & Prior and posterior estimated states. \\
$\boldsymbol P_{(\cdot)}^{-}, \boldsymbol P_{(\cdot)}^{+}$ & Prior and posterior error covariance matrices, with $(\cdot)\in\{ij,io\}$. \\
$\tilde{(\cdot)}, (\cdot)^{*}$ & Measured and desired value of $(\cdot)$, respectively.\\
$\varphi_{i}$ & Desired phase angle of USV $i$.\\
$U_{(\cdot)}\in \mathbb{R}$ & Upper bound of $(\cdot)$.\\
$\boldsymbol{\mu}_{(\cdot)},\boldsymbol{R}_{(\cdot)}$ & Measurement noise and its covariance matrix for measurement $(\cdot)$. The dimension depends on that of $(\cdot)$.\\
\end{tabular}

\end{table}

\section{\textsc{Introduction}}
\IEEEPARstart{T}{arget} tracking is a core capability that enables unmanned surface vehicles (USVs) to conduct missions such as marine environmental monitoring, surface rescue operations, and maritime law enforcement \cite{qiao2023survey, liu2016unmanned,zhangModelReferenceReinforcementLearning2022}. Compared to a single USV, coordinated multi-USV systems demonstrate greater efficiency and robustness in these tracking missions \cite{tong2022cooperative,gaoCoordinatedTargetTracking2021,yanDistributedControlUnmanned2024}. A prevalent coordination strategy for such systems is circumnavigation tracking, where multiple USVs maintain a uniform formation on a circle of a specified radius around the target \cite{yanDistributedControlUnmanned2024,zhengInternalModelObserverbased2024,sunCooperativeStrategyPursuitevasion2022}. Much of the existing studies \cite{che2020cooperative, zhong2019circumnavigation, mwaffo2024formation, chen2025cooperative, liu2025robust, wangReinforcementLearningbasedMovingtarget2024a,shaoFiniteTimeLearningBasedOptimal2025,meiEnhancedFixedTimeCollisionFree2024} focus on controller design for USVs, which typically rely on the assumption that the positions of all USVs and the target are perfectly known. This assumption, however, is challenged by the lack of external positioning infrastructure (\emph{e.g.} in GPS-denied environments), which leads to unreliable localization \cite{zhengInternalModelObserverbased2024,liangDistributedCoordinatedTracking2020, yanDistributedControlUnmanned2024,huBearingOnlyMotionalTargetSurrounding2022}.

Under these constraints, effective circumnavigation requires each USV to estimate both the target state and its relative positions to neighboring vehicles. To tackle this challenge, several solutions have been developed based on limited onboard measurements typically range measurements \cite{hung2020range,yanDistributedControlUnmanned2024,shaoDistanceBasedEllipticalCircumnavigation2023a,hungCooperativeDistributedEstimation2022} or bearing measurements \cite{shames2012circumnavigation,ju2022enclosing,huBearingOnlyMotionalTargetSurrounding2022}. A key limitation of these studies is the ignorance of the intrinsic asymmetry between cooperation and noncooperation in the system, as they assume homogeneous sensing modalities for both targets and neighbors. In reality, however, USVs often cooperate during target tracking, while the target itself is typically noncooperative. This asymmetry leads to distinct sensor configuration challenges across different perception strategies. In particular, bearing measurements can be obtained through passive sensing modalities, such as monocular cameras\cite{suiAdaptiveBearingOnlyTarget2024, huBearingOnlyMotionalTargetSurrounding2022,liRecursiveTotalLeast2025a}. This passivity makes them well-suited for monitoring non-cooperative targets. However, they are ill-suited for inter-USV measurements for two main reasons. First, due to the limited field of view (FOV), monocular cameras can only cover a restricted area, meaning not all neighboring surroundings can be observed. While omnidirectional cameras could mitigate FOV limits, they entail additional burdens, including higher hardware costs and increased calibration complexity, which can be often prohibitive for practical deployment.
Second, bearing measurements become unreliable under non-line-of-sight (NLOS) conditions, thus causing measurement failures like mutual occlusion among USVs. In contrast, range measurements have received significant attention in multi-unmanned systems, especially with the increasing maturity of Ultra-Wideband (UWB) sensors \cite{nguyen2020,liuCooperativeCircumnavigationMultiQuadrotor2025,guoUltraWidebandOdometryBasedCooperative2020,zhaoVehicleCooperativePositioning2025}. Moreover, the omnidirectionality of UWB technology enables it to avoid the limited FOV problem characteristic of vision-based approaches. 

In this paper, we propose a heterogeneous sensing strategy for multi-USV target circumnavigation. Each USV acquires bearing measurements of the target via onboard sensors. A Pseudo-Linear Kalman Filter (PLKF) \cite{liThreeDimensionalBearingOnlyTarget2022} is introduced to estimate the target states. Alternatively, relative positioning between USVs is determined through range measurements. A common approach for estimating high-dimensional relative positions from one-dimensional ranges is to fuse them with inertial odometry information \cite{nguyen2020,nguyen2019,guoUltraWidebandOdometryBasedCooperative2020,zhaoVehicleCooperativePositioning2025}. This fusion enables the construction of a pseudo-linear measurement model by leveraging the spatial geometry between successive odometry displacements and range measurements \cite{nguyen2020,nguyen2019,guoUltraWidebandOdometryBasedCooperative2020}. However, existing studies either disregard the measurement noises \cite{nguyen2020,nguyen2019} or directly assume that UWB measurement noises follow a Gaussian distribution \cite{guoUltraWidebandOdometryBasedCooperative2020}. Recent findings indicate that UWB measurement noises are inherently non-Gaussian and exhibit a heavy-tailed distribution \cite{fishbergMURPMultiAgentUltraWideband2024}. Such non-Gaussian features of UWB measurement noises can lead to significant estimation errors in practice, if not properly accounted for. To address these limitations, we adopt the Maximum Correntropy Kalman Filter (MCKF) \cite{chenMaximumCorrentropyKalman2017}, which mitigates the impact of such noises through the Maximum Correntropy Criterion (MCC), thereby enhancing the reliability of relative localization between USVs. The key difference between the MCKF and the standard Kalman filter (KF) or extended Kalman filter (EKF) lies in their optimization criteria. Conventional filters minimize the Mean Square Error (MSE), an approach that is optimal only under Gaussian noise assumptions. In contrast, MCKF employs the MCC as a robust similarity measure, which effectively downweights large outliers and is thus particularly suited to handling heavy-tailed non-Gaussian noise.

Furthermore, both bearing and range measurements provide only partial positional information. Consequently, Kalman filter-based estimators must satisfy the observability condition to ensure convergent state estimation. This requirement can be met by guaranteeing persistent excitation (PE) of the relative motions among the USVs as well as between each USV and the target \cite{nguyen2020,liThreeDimensionalBearingOnlyTarget2022}. Simultaneously, the multi-USV system must achieve target circumnavigation by forming an evenly spaced circular formation around the target. This leads to a complex dual control problem \cite{deghatLocalizationCircumnavigationSlowly2014,suiAdaptiveBearingOnlyTarget2024}. To address this problem, this paper proposes a two-stage control framework. In the first stage, a coupled oscillator model \cite{liuCooperativeCircumnavigationMultiQuadrotor2025, liuFormationControlEnclosing2025a, liuFormationControlMoving2023} is employed to generate time-varying desired relative positions. This dynamic guidance strategy actively induces the relative motions required to satisfy the PE condition. In the second stage, a classical formation controller is applied by using these desired positions along with the estimated relative states to achieve circumnavigation. Furthermore, the proposed framework is also robust to temporary target occlusion. In such cases, the USVs maintain the enclosing formation by preserving their relative positions with neighboring vehicles.

In summary, unlike previous works \cite{hung2020range, yanDistributedControlUnmanned2024, shaoDistanceBasedEllipticalCircumnavigation2023a, hungCooperativeDistributedEstimation2022, shames2012circumnavigation, ju2022enclosing, huBearingOnlyMotionalTargetSurrounding2022} that assume homogeneous sensing for target circumnavigation, a heterogeneous sensing strategy is proposed to address sensor adaptation issues in practical applications. Regarding relative positioning among USVs, prior range-based studies \cite{nguyen2020, nguyen2019, guoUltraWidebandOdometryBasedCooperative2020} assume PE for estimator convergence without guaranteeing it; in contrast, a cooperative control strategy is actively designed to ensure the PE condition in our work. For bearing-only target localization and tracking, unlike \cite{liThreeDimensionalBearingOnlyTarget2022} which considers a single agent, a multi-USV system is analyzed with a rigorous observability proof. 
Most importantly, the above two objectives are not achieved separately but within a unified framework, simultaneously ensuring both estimator convergence under heterogeneous sensing and the goal of cooperative circumnavigation.
Thus, the main contributions of this work are as follows.
\begin{enumerate}
    \item A general framework is designed for a multi-USV system to perform target circumnavigation under limited onboard sensing, operating without global positioning. It explicitly addresses the inherent sensing asymmetry by estimating the non-cooperative target with passive bearing measurements while localizing neighboring USVs via active ranging, thereby enhancing adaptability to practical constraints.
    
    \item A pseudo-linear Maximum Correntropy Kalman Filter is proposed for range-odometry-based relative localization between USVs. By systematically handling non-Gaussian noise from range sensors, this method overcomes a key limitation in prior work \cite{nguyen2020,nguyen2019,guoUltraWidebandOdometryBasedCooperative2020} and significantly enhances estimation robustness and localization reliability under practical sensing constraints.
    
    \item A dual-purpose cooperative controller is designed that not only organizes the USVs into a stable circular formation but also generates the PE required for estimator observability under partial observations. The theoretical validity of this approach is rigorously established. The strategy also maintains the formation during temporary target occlusion, enhancing robustness in complex scenarios.

\end{enumerate}

The remainder of this paper is structured as follows. Section \ref{section_problem_form} introduces the motion and measurement models of the USVs. Section \ref{section_state_estimate} then details the estimation of relative positions between USVs and the target state. Based on this, Section \ref{section_control} presents the controller design, while Section \ref{section_Analysis} analyzes the observability of the proposed estimators. Finally, Section \ref{section_sim} provides simulation results for verification.

\section{\textsc{PROBLEM FORMULATION}} \label{section_problem_form}
Consider a team of $n$ USVs, denoted by $\mathcal{V}_i$ for $i \in \{1, \ldots, n\}$, which cooperatively circumnavigate a moving target on a horizontal plane, as illustrated in Fig.~\ref{fig:usv_scenario2}.  The sensing and communication assumptions for this task are as follows.

\begin{enumerate} 
    \item Global positioning is unavailable to each USV. However, each vehicle is equipped with an inertial measurement unit (IMU) that provides ego-motion information, enabling periodic odometry displacement measurements.
    \item Each USV can measure relative ranges to all other USVs (\emph{e.g.}, via UWB sensors).
    \item Each USV can obtain bearing measurements of the target (\emph{e.g.}, using visual sensors). During the mission, some USVs may temporarily lose target measurements due to occlusion; however, at any given time, at least one USV in the team maintains a valid measurement of the target.
    \item For each vehicle $\mathcal{V}_i$, the neighbor set is defined as $\mathcal{N}_i={\{1, \ldots, n\}}\backslash\{i\}$, indicating that every USV is mutually interconnected with all others and can exchange essential information accordingly.
\end{enumerate}

Based on the above sensing and communication assumptions, the two main objectives of this work are as follows:

\begin{enumerate}
    \item Estimation of the relative positions between the USVs and the target only using onboard measurements.
    \item Design of a cooperative controller to guide all vehicles to circumnavigate the target at a predefined radius with a uniform spatial distribution and ensure the observability of the relative position estimation.
\end{enumerate}

\begin{figure}[htbp]
  \centering
  \includegraphics[width=0.48\textwidth]{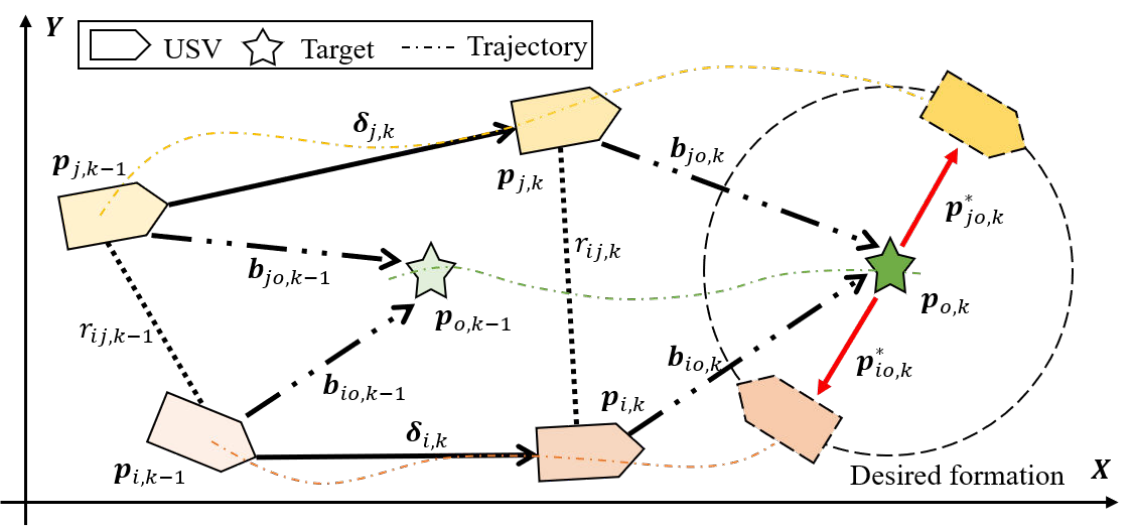}
  \caption{Illustration of the multi-USV cooperative circumnavigation task. At time step $k$, each USV $i$ obtains its own displacement $\boldsymbol{\delta}_{i,k}$, the relative range $r_{ij,k}$ to neighboring USV $j$, and the bearing vector $\boldsymbol{b}_{io,k}$ to the target.} The desired positions of USV $i$ and USV $j$ are denoted by dashed lines, where $\boldsymbol{p}_{io,k}^*$ and $\boldsymbol{p}_{jo,k}^*$ are their respective desired relative position.
  \label{fig:usv_scenario2}
\end{figure}

\subsection{Motion Model}

In this paper, we focus on the fundamental issues of target circumnavigation, specifically relative localization and cooperative control. Hence, a simplified surge-yaw model \cite{yanDistributedControlUnmanned2024} is employed to represent the motion characteristics of the USV \eqref{eq:dynamic_USV}.  considering disturbances or hydrodynamic uncertainties. The discrete surge-yaw kinematic model for each $\mathcal{V}_i$ is described as follows
    \begin{equation}\label{eq:dynamic_USV}
        \begin{aligned}
          & \boldsymbol{p}_{i,k+1} = \boldsymbol{p}_{i,k} + \begin{bmatrix} \cos \theta_{i,k} \\ \sin \theta_{i,k} \end{bmatrix} v_{i,k} \Delta t  \\ 
          & \theta_{i,k+1} = \theta_{i,k} + \omega_{i,k} \Delta t
        \end{aligned}
    \end{equation}
 where $\boldsymbol{p}_{i,k} \in \mathbb{R}^2$ and $\theta_{i,k} \in \mathbb{R}$ denote the position and yaw angle of $\mathcal{V}_i$ at time $k\Delta t$, respectively, with $\Delta t$ being the sampling period; $v_{i,k} \in \mathbb{R}$ and $\omega_{i,k} \in \mathbb{R}$ are the control inputs, representing the surge speed and yaw rate, respectively. According to \eqref{eq:dynamic_USV}, the velocity of $\mathcal{V}_i$ at time step $k$ is given by 
    \begin{equation}\label{eq:USV_wel}
        \boldsymbol{v}_{i,k} = v_{i,k} \begin{bmatrix} \cos \theta_{i,k} \\ \sin \theta_{i,k} \end{bmatrix}
    \end{equation}
A real USV system generally has limited speed and turning rate. Hence, the control inputs are assumed to be bounded with $\left\vert v_{i,k} \right\vert  \leq U_{\boldsymbol{v}_{i}}$ and $ \left\vert \omega_{i,k} \right\vert  \leq U_{\omega_{i}}$, where $U_{\boldsymbol{v}_{i}}$ and $U_{\omega_{i}}$ are positive constants.

We denote the target state as $\boldsymbol{x}_{o,k} = \begin{bmatrix} \boldsymbol{p}_{o,k}^T  & \boldsymbol{v}_{o,k}^T \end{bmatrix}^T \in \mathbb{R}^4$, where $\boldsymbol{p}_{o,k} \in \mathbb{R}^2$ represents the target position, and $\boldsymbol{v}_{o,k} \in \mathbb{R}^2$ represents the target velocity. For the non-cooperative moving target, this paper employs a noise-driven double-integrator model \cite{liThreeDimensionalBearingOnlyTarget2022,zhengOptimalSpatialTemporal2025}
    \begin{equation}\label{eq:dynamic_target}
        \boldsymbol{x}_{o,k} = \boldsymbol{A} \boldsymbol{x}_{o,k-1} + \boldsymbol{w}_{o,k-1}
    \end{equation}
where $\boldsymbol{w}_{o,k} \sim \mathcal{N}(\boldsymbol{0}, \boldsymbol{Q}_{o})$ are the zero-mean Gaussian process noises with $\boldsymbol{Q}_{o}$ denoting the covariance matrix, and
\begin{equation}
\quad \boldsymbol{A} = \begin{bmatrix} \boldsymbol{I}_{2} & \Delta t \boldsymbol{I}_{2} \\ \boldsymbol{0}_{2} & \boldsymbol{I}_{2} \end{bmatrix}
\end{equation}
The process noise $\boldsymbol{w}_{o,k}$ is introduced to characterize the unknown maneuverability of the target. The magnitude of the target velocity $\boldsymbol{v}_{o,k}$ is assumed to be bounded by a constant $U_{\boldsymbol{v}_{o}} >0$.

\subsection{Onboard Measurement Model}

\textit{1) Odometry-based displacement measurements:} As a core capability, autonomous navigation enables a USV to estimate its ego-motion by fusing data from proprioceptive sensors (\emph{e.g.} IMUs) with exteroceptive sensors (\emph{e.g.} vision, LiDAR). In this context, the self-displacement  $\tilde{\boldsymbol{\delta}}_{i,k}$ of the USV is measured at a fixed time intervals $\Delta t$. The measurement model is defined as
\begin{equation}\label{eq:self_displacement}
\left\{\begin{array}{l}
    \tilde{\boldsymbol{\delta}}_{i,k} = \boldsymbol{\delta}_{i,k} + \boldsymbol{\mu}_{\delta_{i},k} \\
     \boldsymbol{\delta}_{i,k} = \boldsymbol{p}_{i,k} - \boldsymbol{p}_{i,k-1}
\end{array}
\right.
\end{equation}
where $\boldsymbol{\delta}_{i,k}$ denote the true displacement values, and $\boldsymbol{\mu}_{\boldsymbol\delta_{i},k} \sim \mathcal{N}(\boldsymbol{0}, \boldsymbol{R}_{\boldsymbol\delta_{i}})$ represent the zero-mean Gaussian measurement noises with the covariance matrix given by $\boldsymbol{R}_{\boldsymbol\delta_{i}}$.

\textit{2) Inter-USV range measurements:} It is assumed that each USV is capable of measuring the relative distance to their neighboring vehicles. Formally, the range measurement $\tilde{r}_{ij,k}$ between neighbouring vehicles $\mathcal{V}_i$ and $\mathcal{V}_j$, with $i,j\in \{1, \ldots, n\}$, is specified as
\begin{equation}\label{eq:relative_range}
    \left\{\begin{array}{l}
    \tilde{r}_{ij,k}= {r}_{ij,k} + {\mu}_{r_{ij},k} \\
     {r}_{ij,k} = \|\boldsymbol{p}_{i,k}-\boldsymbol{p}_{j,k}\|
    \end{array}
    \right.
\end{equation}
where ${r}_{ij,k}$ denote the true inter-USV range measurements, and ${\mu}_{r_{ij},k}$ represent the measurement noises. In real-world applications, the inter-USV range measurements can be acquried using UWB sensors \cite{liuCooperativeCircumnavigationMultiQuadrotor2025,nguyen2020,fishbergMURPMultiAgentUltraWideband2024}. However, as noted in \cite{liuCooperativeCircumnavigationMultiQuadrotor2025,fishbergMURPMultiAgentUltraWideband2024}, a large body of current work assumes that UWB measurement noise follows a Gaussian distribution, despite the presence of outliers in actual sensor noise. To address this, this paper models the measurement noise using a Gaussian mixture model to account for such outliers. 
\begin{equation}\label{eq:r_noise_model_multi_Gaus}
    {\mu}_{r_{ij},k} \sim (1-\epsilon)\mathcal{N}(0, \boldsymbol{R}_{{\mu}_{r_{ij}}}) + \epsilon\mathcal{N}(0,\boldsymbol{R}_{{\mu}_{r_{ij}}}^{outliers})
\end{equation}
where $\epsilon$ denotes the proportion of outliers, $\boldsymbol{R}_{{\mu}_{r_{ij}}}$ is the nominal covariance and $\boldsymbol{R}_{{\mu}_{r_{ij}}}^{outliers}$ is the covariance associated with the outlier component.
Furthermore, as reported in \cite{fishbergMURPMultiAgentUltraWideband2024}, the analysis of real UWB sensor data reveals that the measurement noise exhibits an asymmetric and heavy-tailed distribution. Therefore, this paper also considers a log-normal distribution model for the measurement noise.
\begin{equation}\label{eq:r_noise_model_logGaus}
      \ln \left({\mu}_{r_{ij},k} + r_s\right) \sim \mathcal{N}(0, \boldsymbol{R}_{{\mu}_{r}}) 
\end{equation}
where $\boldsymbol{R}_{{\mu}_{r}}$ represents the equivalent covariance and $r_s$ represents the shift parameter, which is used to adjust the degree of left skewness.
In Section III-A, the common characteristics of the noise in the final equivalent measurement values under these two different models will be further analyzed.

\textit{3) Bearing measurement of target:} The bearing vector from the target to the USV $\mathcal{V}_i$ is defined as \cite{liThreeDimensionalBearingOnlyTarget2022,zhengOptimalSpatialTemporal2025}
\begin{equation}\label{eq:target_bearing}
        \left\{\begin{array}{l}
    \tilde{\boldsymbol{b}}_{io,k} = \boldsymbol{b}_{io,k} + \boldsymbol{\mu}_{\boldsymbol{b}_{io},k} \\
     \boldsymbol{b}_{io,k} = \frac{\boldsymbol{p}_{i,k}-\boldsymbol{p}_{o,k}}{\left\|\boldsymbol{p}_{i,k}-\boldsymbol{p}_{o,k}\right\|}
    \end{array}
    \right.
\end{equation}
where $\boldsymbol{b}_{io,k}$ denote the true bearing vectors,  and $\boldsymbol{\mu}_{\boldsymbol{b}_{io},k} \sim \mathcal{N}(\boldsymbol{0}, \boldsymbol{R}_{\boldsymbol{b}_{io}})$ are the zero-mean Gaussian measurement noise with the covariance matrix defined by $\boldsymbol{R}_{\boldsymbol{b}_{io}}$. In practice, bearing information can be acquired using vision-based sensors, such as a servo camera \cite{liThreeDimensionalBearingOnlyTarget2022,zhengOptimalSpatialTemporal2025}. Consequently, the USV is assumed to be capable of consistently obtaining these measurements from the target, except in cases where the line of sight is obstructed by obstacles.

\section{\textsc{STATE ESTIMATION}} \label{section_state_estimate}
In this section, we introduce the filtering algorithms to achieve the relative localization among USVs and estimate the target state, respectively. Firstly, a MCKF \cite{chenMaximumCorrentropyKalman2017} is developed for the relative localization among USVs using range and displacement measurements. Secondly, a PLKF \cite{liThreeDimensionalBearingOnlyTarget2022} is designed to estimate the target state using bearing-only measurements.

\subsection{Relative Localization Inter-USVs}

Denote the relative position between $\mathcal{V}_i$ and $\mathcal{V}_j$ as
    \begin{equation} \label{eq:relative_pos}
        \boldsymbol{p}_{ij,k} = \boldsymbol{p}_{i,k} - \boldsymbol{p}_{j,k}
    \end{equation}
As established in \cite{liuFormationControlEnclosing2025a}, spatial geometric constraints derived from consecutive relative distance and displacement measurements lead to the following model via the cosine law
    \begin{equation} \label{eq:without_noise_meas}
        \frac{1}{2} \left( r_{ij,k}^{\,2} - r_{ij,k-1}^{\,2} + \| {\boldsymbol{\delta}}_{ij,k} \|^2 \right)
        = \boldsymbol{\delta}_{ij,k}^T \boldsymbol{p}_{ij,k}
    \end{equation}
with $\boldsymbol{\delta}_{ij,k}= \boldsymbol{\delta}_{i,k} - \boldsymbol{\delta}_{j,k}$ denoting the relative displacement between $\mathcal{V}_i$ and $\mathcal{V}_j$. For $\mathcal{V}_i$, it can receive the displacement $\boldsymbol{\delta}_{j,k}$ from $\mathcal{V}_j$ via inter-vehicle communication. Since the ideal value in \eqref{eq:without_noise_meas} is corrupted by noise in practice, we formulate the measurement model as 
    \begin{equation} \label{eq:rel_meas}
         z_{ij,k} = \tilde{\boldsymbol{\delta}}_{ij,k}^T
         \boldsymbol{p}_{ij,k} + \mu_{z_{ij},k}
    \end{equation}
where 

\begin{equation} \label{eq:rel_meas_}
    \left\{
    \begin{array}{l}
         z_{ij,k} = \frac{1}{2} \left( \tilde r_{ij,k}^{\,2} - \tilde r_{ij,k-1}^{\,2} + \| \tilde{\boldsymbol{\delta}}_{ij,k} \|^2 \right)\\
         \tilde{\boldsymbol{\delta}}_{ij,k} = \tilde{\boldsymbol{\delta}}_{i,k} - \tilde{\boldsymbol{\delta}}_{j,k} = \boldsymbol{\delta}_{ij,k} + \boldsymbol{\mu}_{\delta_{ij},k}
    \end{array} 
        \right.
    \end{equation}
with $\boldsymbol{\mu}_{\delta_{ij},k} \sim \mathcal{N}(\boldsymbol{0}, \boldsymbol{R}_{\delta_{ij}})$ is the zero-mean Gaussian measurement noise with the covariance matrix given by $\boldsymbol{R}_{\delta_{ij}} = \boldsymbol{R}_{\delta_{i}} + \boldsymbol{R}_{\delta_{j}}$, and $\mu_{z_{ij},k}$ are the equivalent measurement noises with a zero mean and a covariance matrix $\boldsymbol{R}_{z_{ij},k}$. 

Note that the equivalent measurement noises $\mu_{z_{ij},k}$ in model \eqref{eq:rel_meas} are compound noises with complex properties. The displacement measurement noises from an IMU are typically much smaller than range measurement noise over short intervals, so they are assumed to be negligible. The equivalent measurement noises $\mu_{z_{ij},k}$ arise from the combined effect of linear and quadratic terms. Through derivation, it can be shown that 
\begin{equation} \label{eq:noise_approximation}
    \begin{aligned}
    \mu_{z_{ij},k} \approx &r_{ij,k} \mu_{r_{ij},k} - r_{ij,k-1} \mu_{r_{ij},k-1} \\
    &+ \frac{1}{2} \left(\mu_{r_{ij},k}^2 - \mu_{r_{ij},k-1}^2\right)
    \end{aligned}
\end{equation}
In practical scenarios, the relative distance $r_{ij}$ between USVs typically undergoes only slight variations over a short sampling interval. Consequently, the statistical characteristics of the equivalent noise $\mu_{z_{ij},k}$ are predominantly governed by the terms $(\mu_{r_{ij},k} - \mu_{r_{ij},k-1})$ and $(\mu_{r_{ij},k}^2 - \mu_{r_{ij},k-1}^2)$. Since these terms involve subtracting two identically distributed random variables, structural analysis reveals that, regardless of whether the range measurement noise ${\mu}_{r_{ij},k}$ follows model \eqref{eq:r_noise_model_multi_Gaus} or model \eqref{eq:r_noise_model_logGaus}, the equivalent noise $\mu_{z_{ij},k}$ exhibits the following statistical properties: (i) it is approximately symmetric about zero; (ii) it has a sharp peak near zero; and (iii) it displays heavy-tailed behavior on both sides. 

For robust relative localization between USVs under the non-Gaussian noise in measurement model \eqref{eq:rel_meas}, this paper implements the MCKF. This algorithm addresses the limitations of conventional Kalman filters in heavy-tailed or impulsive noise by adopting the maximum correntropy criterion (MCC). Unlike the minimum mean square error criterion, which uses a quadratic cost function sensitive to outliers, MCC evaluates similarity through a Gaussian kernel-based measure of joint probability density. This approach accounts for higher-order statistics, inherently suppressing large errors and improving estimation accuracy in challenging noise conditions. 

The flowchart of the MCKF-based relative positioning algorithm is presented below, while detailed theoretical foundations can be found in \cite{chenMaximumCorrentropyKalman2017}.

1) Let the prior and posterior estimates of $\boldsymbol{p}_{ij,k}$ at time step $k$ be denoted as $\hat{\boldsymbol p}_{ij,k}^-$ and $\hat{\boldsymbol p}_{ij,k}^+$, respectively, with corresponding error covariances $\boldsymbol P_{ij,k}^-$ and $\boldsymbol P_{ij,k}^+$. For initialization, select appropriate Gaussian kernel bandwidths $\boldsymbol \sigma = [\sigma_1,\sigma_2,\sigma_3]$ and a sufficiently small tolerance $\iota > 0$. The algorithm is initialized with a given initial guess $(\hat{\boldsymbol p}_{ij,0}^+, \boldsymbol P_{ij,0}^+ )$.

2) For $k \geq 1$, USV $\mathcal{V}_i$ obtains its own displacement $\tilde{\boldsymbol{\delta}}_{i,k}$ from ego-motion and sends it to neighbor $\mathcal{V}_j$; simultaneously, it receives $\tilde{\boldsymbol{\delta}}_{j,k}$ form $\mathcal{V}_j$. The relative position is then propagated as follows
\begin{equation} \label{eq:propogation_ij}
    \begin{aligned}
    &\hat{\boldsymbol p}_{ij,k}^-=\hat{\boldsymbol p}_{ij,k-1}^+ + \tilde{\boldsymbol{\delta}}_{ij,k} \\
    &\boldsymbol{P}_{ij,k}^- = \boldsymbol{P}_{ij,k-1}^+ + \boldsymbol{R}_{\delta_{ij}}
    \end{aligned}
\end{equation}

3) Perform Cholesky decompositions on $\boldsymbol{P}_{ij,k}^-$ and $\boldsymbol{R}_{z_{ij},k}$ to obtain $\boldsymbol{B}_{\boldsymbol p_{ij},k}$ and $\boldsymbol{B}_{z_{ij},k}$, respectively, such that
\begin{equation} \label{eq:Cholesky}
    \boldsymbol{P}_{ij,k}^- = \boldsymbol{B}_{\boldsymbol p_{ij},k} \boldsymbol{B}_{\boldsymbol p_{ij},k}^T, \quad
    \boldsymbol{R}_{z_{ij},k} = \boldsymbol{B}_{z_{ij},k} \boldsymbol{B}_{z_{ij},k}^T
\end{equation}
Then, construct the augmented block matrix
\begin{equation}
    \boldsymbol{B}_{ij,k} = \begin{bmatrix} \boldsymbol{B}_{\boldsymbol p_{ij},k} & \mathbf{0} \\ \mathbf{0} & \boldsymbol{B}_{z_{ij},k} \end{bmatrix}
\end{equation}
Further, compute the following matrices 
\begin{equation}
    \boldsymbol{D}_{ij,k} = \boldsymbol{B}_{ij,k}^{-1} 
\begin{bmatrix} \hat{\boldsymbol{p}}_{ij,k}^- \\ z_{ij,k} \end{bmatrix}, \quad
 \boldsymbol{W}_{ij,k} = \boldsymbol{B}_{ij,k}^{-1} 
\begin{bmatrix} \boldsymbol{I}_2 \\ \tilde{\boldsymbol{\delta}}_{ij,k}^T \end{bmatrix}
\end{equation}

4) The posterior estimate $\hat{\boldsymbol p}_{ij,k}^+$ at the time step $k$ is obtained as the convergent value of the iterative sequence $\hat{\boldsymbol p}_{ij,k}^{(s)}$. The sequence is initialized with $\hat{\boldsymbol p}_{ij,k}^{(0)} = \hat{\boldsymbol p}_{ij,k}^-$ and updated for $s>0$ by
\begin{equation} \label{eq:MCKF_iteration}
    \hat{\boldsymbol{p}}_{ij,k}^{(s)} = \hat{\boldsymbol p}_{ij,k}^- + 
    \bar {\boldsymbol{K}}_{ij,k}\bigl(z_{ij,k} - \tilde{\boldsymbol{\delta}}_{ij,k}^T\hat{\boldsymbol{p}}_{ij,k}^-\bigr)
\end{equation}
where the gain matrix $\bar {\boldsymbol{K}}_{ij,k}$ and related terms are given by
\begin{equation*}
\left\{
    \begin{aligned}
    &\bar{\boldsymbol K}_{ij,k}=\bar{\boldsymbol P}_{ij,k}^{-} \tilde{\boldsymbol{\delta}}_{ij,k}
            \big(\tilde{\boldsymbol{\delta}}_{ij,k}^T \bar{\boldsymbol P}_{ij,k}^{-} \tilde{\boldsymbol{\delta}}_{ij,k} + \bar{\boldsymbol R}_{ij,k}\big)^{-1} \\
    &\bar{\boldsymbol P}_{ij,k} ^{-}= \boldsymbol{B}_{\boldsymbol p_{ij},k} (\bar{\boldsymbol{C}}_{\boldsymbol p_{ij},k})^{-1} \boldsymbol{B}_{\boldsymbol p_{ij},k}^T \\
    &\bar{\boldsymbol R}_{ij,k} ^{-}= \boldsymbol{B}_{z_{ij},k} (\bar{\boldsymbol{C}}_{z_{ij},k})^{-1} \boldsymbol{B}_{z_{ij},k}^T \\
    &\bar{\boldsymbol{C}}_{\boldsymbol p_{ij},k} = \mathrm{diag}\left[G_{\sigma_1}(w_{ij,k}^1),G_{\sigma_2}(w_{ij,k}^2)\right]  \\
    &\bar{\boldsymbol{C}}_{z_{ij},k} = \mathrm{diag}\left[G_{\sigma_3}(w_{ij,k}^3)\right] \\
    &\boldsymbol{w}_{ij,k} = \boldsymbol{D}_{ij,k} - \boldsymbol{W}_{ij,k} \hat{\boldsymbol{p}}_{ij,k}^{(s-1)}
    \end{aligned}
    \right.
\end{equation*}
Here, $w_{ij,k}^l$ (for $l={1,2,3}$) denotes the $l$-th element of $\boldsymbol{w}_{ij,k}$, and 
\begin{equation*}
    G_{\sigma_l}(w_{ij,k}^l) = \exp\!\left(-\frac{\big(w_{ij,k}^l\big)^2}{2\,\sigma_l^2}\right)
\end{equation*}
where $G_{\sigma_l}(w_{ij,k}^l)$ is the Gaussian kernel function.

5) The iteration stops once the condition in \eqref{eq:stop_iteration} is satisfied, and the posterior estimate is assigned as $\hat{\boldsymbol p}_{ij,k}^+ = \hat{\boldsymbol p}_{ij,k}^{(s)}$. Otherwise, the index $s$ is incremented, and the iteration continues. To ensure real-time computational determinism, a maximum number of iterations $L_{\max}$ is further imposed as a secondary stopping criterion.
\begin{equation} \label{eq:stop_iteration}
    \frac{\|\hat{\boldsymbol{p}}_{ij,k}^{(s)} - \hat{\boldsymbol{p}}_{ij,k}^{(s-1)}\|}
    {\|\hat{\boldsymbol{p}}_{ij,k}^{(s-1)}\|} \leq \iota
\end{equation}

6) Update the posterior covariance matrix using
\begin{equation} \label{eq:posterior_ij}
    \begin{aligned}
    \boldsymbol{P}_{ij,k}^+ = & (\boldsymbol{I} - \bar{\boldsymbol{K}}_{ij,k}\tilde{\boldsymbol{\delta}}_{ij,k}^T) \boldsymbol{P}_{ij,k}^-(\boldsymbol{I} - \bar {\boldsymbol{K}}_{ij,k}\tilde{\boldsymbol{\delta}}_{ij,k}^T)^T \\
    &+ \bar {\boldsymbol{K}}_{ij,k} {\boldsymbol{R}}_{ij,k}\bar {\boldsymbol{K}}_{ij,k}^T 
    \end{aligned}
\end{equation}
With this, the estimation cycle for time step $k$ is complete. The algorithm then proceeds to time step $k+1$ to acquire new measurements and repeat the updating process.

The kernel bandwidth $\sigma_l$ has a significant influence on the performance of the MCKF algorithm. As $\sigma_l$ approaches positive infinity, the MCKF reduces to the classical Kalman filter. Hence, a large $\sigma_l$ value should be selected, when the measurement noise exhibits mild heavy-tailed characteristics. Conversely, for significant heavy-tailed measurement noises, a small $\sigma_l$ value are suggested to enhance the filtering robustness against outliers.

\subsection{USV–Target Relative State Estimation}
The relative state $\boldsymbol{x}_{io,k}$ between USV $\mathcal{V}_i$ and the target, derived from motion models \eqref{eq:dynamic_USV} and \eqref{eq:dynamic_target}, is defined by
$$ \boldsymbol{x}_{io,k} \triangleq \begin{bmatrix}  \boldsymbol{p}_{io,k} \\ \boldsymbol{v}_{io,k} \end{bmatrix} = \begin{bmatrix}  \boldsymbol{p}_{i,k} - \boldsymbol{p}_{o,k} \\ \boldsymbol{v}_{i,k} - \boldsymbol{v}_{o,k} \end{bmatrix} $$
which comprises the relative position $\boldsymbol{p}_{io,k}$ and relative velocity $\boldsymbol{v}_{io,k}$.

Following the PLKF in \cite{liThreeDimensionalBearingOnlyTarget2022}, the pseudo-linear bearing measurement equation is formulated as
\begin{equation} \label{eq:bearing_PL}
    \boldsymbol{z}_{io,k} \triangleq \mathbf{0}_{2 \times 1} = \boldsymbol{G}_{\tilde{\boldsymbol{b}}_{io},k} \boldsymbol{p}_{io,k} + \boldsymbol{\mu}_{\boldsymbol{z}_{io},k}
\end{equation}
Here, $\boldsymbol{G}_{\tilde{\boldsymbol{b}}_{io},k} = \boldsymbol{I}_{2 \times 2} - \tilde{\boldsymbol{b}}_{io,k} \tilde{\boldsymbol{b}}_{io,k}^T $ is an orthogonal projection operator, and $\boldsymbol{\mu}_{\boldsymbol{z}_{io},k}$ is the equivalent measurement noise given by $\boldsymbol{\mu}_{\boldsymbol{z}_{io},k} = {r}_{io,k} \boldsymbol{G}_{\tilde{\boldsymbol{b}}_{io},k} \boldsymbol{\mu}_{\boldsymbol{b}_{io},k}$, where ${r}_{io,k}=\|\boldsymbol{p}_{o,k} - \boldsymbol{p}_{i,k}\|$ denotes the true distance between $\mathcal{V}_i$ and the target. 
Fig.~\ref{fig:Graphical illustration} illustrates the geometric relationship between the orthogonal projection operator $\boldsymbol{G}_{\boldsymbol{b}{io},k}$ and the relative position vector $\boldsymbol{p}_{io,k}$. The plane spanned by $\boldsymbol{G}_{\boldsymbol{b}{io},k}$ is orthogonal to $\boldsymbol{b}_{io,k}$. Since $\boldsymbol{b}_{io,k}$ represents the unit direction vector of $\boldsymbol{p}_{io,k}$, it follows that $\boldsymbol{G}_{\boldsymbol{b}_{io},k}$ is also orthogonal to $\boldsymbol{p}_{io,k}$, thereby satisfying $\boldsymbol{G}_{\boldsymbol{b}{io},k} \boldsymbol{p}_{io,k} = \mathbf{0}$. However, due to measurement noise, $\boldsymbol{G}_{\tilde{\boldsymbol{b}}_{io},k}$ is not perfectly orthogonal to $\boldsymbol{p}_{io,k}$, leading to a deviation captured by the measurement model in \eqref{eq:bearing_PL}.

\begin{figure}[tbp]
  \centering
  \includegraphics[width=0.3\textwidth]{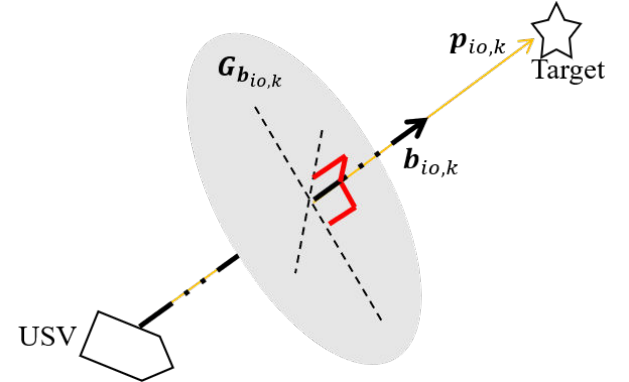}
  \caption{Schematic diagram illustrating the geometric relationship where the bearing vector ${\boldsymbol{b}}_{io,k}$ is orthogonal to its projection matrix $\boldsymbol{G}_{{\boldsymbol{b}}_{io,k}} = \boldsymbol{I}  - \boldsymbol{b}_{io,k}\boldsymbol{b}_{io,k}^T $.}
  \label{fig:Graphical illustration}
\end{figure}

To estimate the relative state $\boldsymbol{x}_{io,k}$, we formulate the final measurement equation via matrix augmentation as
\begin{equation} \label{eq:bearing_state_PL}
    \boldsymbol{z}_{io,k} \triangleq\mathbf{0}_{2 \times 1} =  \begin{bmatrix}  \boldsymbol{G}_{\tilde{\boldsymbol{b}}_{io},k} & \mathbf{0}_{2 \times 2} \end{bmatrix} \boldsymbol{x}_{io,k} + \boldsymbol{\mu}_{\boldsymbol{z}_{io},k}
\end{equation}

Thus, the propagation step for the state estimate $\boldsymbol{x}_{io,k}$ is performed as follows
\begin{equation} \label{eq:propogation_oi}
    \begin{aligned}
    &\hat{\boldsymbol x}_{io,k}^- = \boldsymbol{A} \hat{\boldsymbol x}_{io,k-1}^+  \\
    &\boldsymbol{P}_{io,k}^- = \boldsymbol{A} \boldsymbol{P}_{io,k-1}^+ \boldsymbol{A}^T + \boldsymbol{Q}_{o}
    \end{aligned}
\end{equation}
Here, $\hat{\boldsymbol x}_{io,k}^-$ and $\hat{\boldsymbol x}_{io,k}^+$ are the prior and posterior state estimates, and $\boldsymbol{P}_{io,k}^-$ and $\boldsymbol{P}_{io,k}^+$ are their corresponding error covariance matrices.

Since the true distance ${r}_{io,k}$ is unavailable, the predicted distance $\hat{r}_{io,k}^- = \| \left[\boldsymbol{I} \quad  \mathbf{0}\right] \hat{\boldsymbol x}_{io,k}^-\|$ is used instead, which yields the noise covariance $\boldsymbol{R}_{\boldsymbol{z}_{io},k} = (\hat{r}_{io,k}^-)^2 \boldsymbol{G}_{\tilde{\boldsymbol{b}}_{io},k} \boldsymbol{R}_{\boldsymbol{b}_{io}} \boldsymbol{G}_{\tilde{\boldsymbol{b}}_{io},k}^T$ of $\boldsymbol{\mu}_{\boldsymbol{z}_{io},k}$. With $\boldsymbol{H}_{io,k} \triangleq \begin{bmatrix}  \boldsymbol{G}_{\tilde{\boldsymbol{b}}_{io},k} & \mathbf{0}_{2 \times 2} \end{bmatrix}$, the correction update is then performed
\begin{equation} \label{eq:update_oi}
    \begin{aligned}
    &\boldsymbol{K}_{io,k} = \boldsymbol{P}_{io,k}^- \boldsymbol{H}_{io,k}^T \left( \boldsymbol{H}_{io,k} \boldsymbol{P}_{io,k}^- \boldsymbol{H}_{io,k}^T + \boldsymbol{R}_{\boldsymbol{z}_{io},k} \right)^\dagger  \\
    &\hat{\boldsymbol x}_{io,k}^+ = \hat{\boldsymbol{x}}_{io,k}^{-} - \boldsymbol{K}_{io,k} \boldsymbol{H}_{io,k} \hat{\boldsymbol{x}}_{io,k}^{-} \\
    &\boldsymbol{P}_{io,k}^+ = \left( \boldsymbol{I} - \boldsymbol{K}_{io,k} \boldsymbol{H}_{io,k} \right) \boldsymbol{P}_{io,k}^- 
    \end{aligned}
\end{equation}
where $\dagger$ represents the pseudoinverse.

It should be noted that vision-based bearing measurements may fail due to obstruction by obstacles. In this case, the PLKF reverts to a prediction-only step
\begin{equation} \label{eq:predic_only}
    \hat{\boldsymbol{x}}_{io,k}^{+}=\hat{\boldsymbol{x}}_{io,k}^{-}, \qquad \boldsymbol{P}_{io,k}^+=\boldsymbol{P}_{io,k}^-
\end{equation}
This skips the measurement update to prevent the injection of spurious data. The resulting increase in $\boldsymbol{P}_{io,k}^+$ properly reflects the growing state uncertainty. Standard PLKF updates resume once valid bearing measurements are restored.

\section{\textsc{Circumnavigation Control}} \label{section_control}
To achieve coordinated circumnavigation, a two-stage control strategy is designed for each USV using the relative estimates from the previous section. This strategy first employs a coupled oscillator model to autonomously negotiate desired relative positions, then applies a formation control law to realize the maneuver. Additionally, a vector-field-based obstacle avoidance strategy is incorporated for safe navigation.

An overview of the proposed heterogeneous perception and control architecture for each USV is illustrated in Fig.~\ref{fig:framework}. Specifically, for vehicle $\mathcal{V}_i$, the relative inter-USV localization is estimated using the MCKF, based on the received information $\boldsymbol{\delta}_{j,k}$ from $\mathcal{V}_j$ via communication network. Meanwhile, the relative state between the USV and the target is estimated using the PLKF. Concurrently, $\mathcal{V}_i$ receives the phase angle $\varphi_{j,k}$ from neighboring USVs through communication, which is then used in the coupled oscillator model to generate the reference relative positions for circular formation. These estimates, along with the reference formation geometry, are subsequently fed into the circumnavigation control module. This architecture effectively addresses the asymmetry inherent in heterogeneous sensing while ensuring stable circumnavigation performance in complex environments.

\begin{figure}[tbp]
  \centering
  \includegraphics[width=0.47\textwidth]{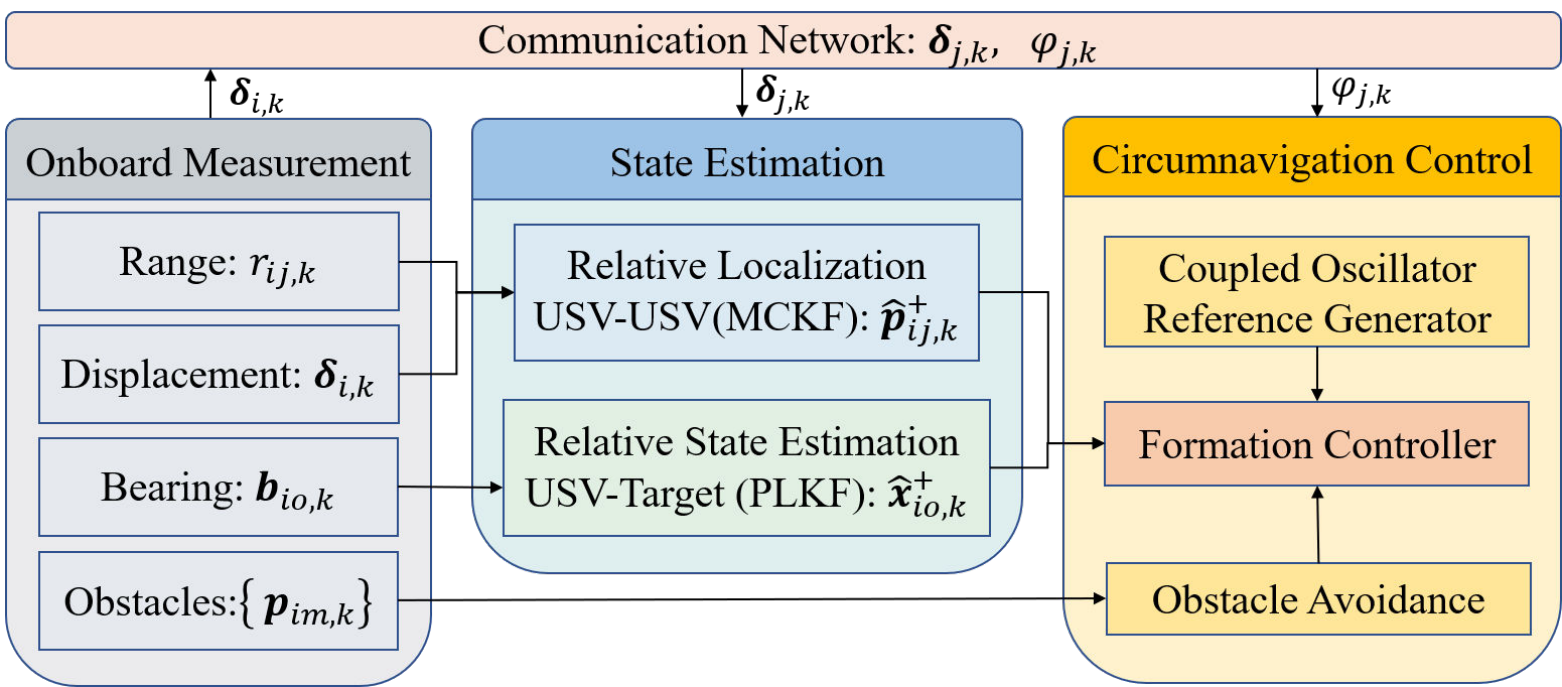}
  \caption{System architecture of the proposed USV circumnavigation framework. Each USV relies solely on onboard measurements to achieve relative state estimation for both USV-USV and USV-target interactions, and then implements circumnavigation control based on the estimated values and the desired references generated by a coupled oscillator model. The USVs exchange critical information with one another via a communication network.}
  \label{fig:framework}
\end{figure}

\subsection{Desired Relative Positions Design}
The desired relative position between USV $\mathcal{V}_i$ and the target at time step $k$ is defined as
    \begin{equation}\label{eq:desired_pos}
        \boldsymbol{p}_{io,k}^* = r_{io}^*  \left[ \begin{array}{cc} \cos{\varphi_{i,k}} \\  \sin{\varphi_{i,k}} \\ \end{array} \right]
    \end{equation}
where $r_{io}^*$ is the specified circumnavigation radius and $\varphi_{i,k}$ is the desired phase angle. To achieve autonomous negotiation of a uniform distribution around the target, a coupled oscillator model \cite{liuFormationControlMoving2023,dorflerSynchronizationComplexNetworks2014} is employed, wherein neighboring USVs exchange their desired phase angles. The phase dynamics are governed by
	\begin{equation}\label{eq:osc}
		\begin{split}
			\varphi_{i,k+1} &= \varphi_{i,k} + \omega_{i}^*\Delta t \\ &+ \Delta t \sum_{j = 1}^{n} \sum_{m = 1}^{n}  \frac{K_{m}}{mn}\sin{\left(m\left[\varphi_{i,k}-\varphi_{j,k}\right]\right)}
		\end{split}
	\end{equation}
where $K_{m}\in\mathbb{R}$ (for $m\in \{1,\ldots,n \}$) are appropriate coupling gains, and $\omega_{i}^*>0$ denotes the constant desired yaw rate. Thus, the corresponding desired relative velocity between USV $\mathcal{V}_i$ and the target is
    \begin{equation}\label{eq:desired_vel_i0}
        \boldsymbol{v}_{io,k}^* = \frac{1}{\Delta t} (\boldsymbol{p}_{io,k+1}^* - \boldsymbol{p}_{io,k}^*)
    \end{equation}
Furthermore, the desired relative position between $\mathcal{V}_i$ and $\mathcal{V}_j$ is given by
    \begin{equation}\label{eq:desired_pos_ij}
        \boldsymbol{p}_{ij,k}^{*} = \boldsymbol{p}_{io,k}^{*} - \boldsymbol{p}_{jo,k}^{*}
    \end{equation}

Under the assumption that each USV $\mathcal{V}_i$ communicates mutually with all its neighbors, the system \eqref{eq:osc} converges exponentially fast to a stable state \cite[Theorem 5.4]{dorflerSynchronizationComplexNetworks2014}. Furthermore, when the system encounters dynamic changes such as USVs joining or leaving, it can promptly re-converge to a new equilibrium, thereby maintaining a uniform distribution of USVs around the target \cite{liuFormationControlMoving2023,liuCooperativeCircumnavigationMultiQuadrotor2025}.

\subsection{Cooperative Formation Control}
Based on the estimated relative positions $\hat{\boldsymbol p}_{ij,k}^+$, relative states $\hat{\boldsymbol{x}}_{io,k}^+ \triangleq \begin{bmatrix} (\hat{\boldsymbol{p}}_{io,k}^+)^T & (\hat{\boldsymbol{v}}_{io,k}^+)^T \end{bmatrix}^T$, and the desired relative positions, a displacement-based formation controller \cite{ohSurveyMultiagentFormation2015,liuFormationControlMoving2023} for each USV $\mathcal{V}_i$ is designed as follows
\begin{equation} \label{eq:controller}
\left\{
    \begin{aligned}
    &v_{i,k} = \begin{bmatrix} \cos \theta_{i,k} & \sin \theta_{i,k} \end{bmatrix} \boldsymbol{v}_{i,k}^c  \\
    &\omega_{i,k} = \text{sat}(U_{\omega_i^c},\omega_{i,k}^c) + \omega_{i}^*
    \end{aligned}
    \right.
\end{equation}
Here, $\text{sat}(U,\boldsymbol{u}) = \left( \min \{U,\| \boldsymbol{u}\|\}/\|\boldsymbol{u}\| \right) \boldsymbol{u}$ defines a saturation function, and $U_{\omega_i^c}>0$ bounds the correction yaw rate $\omega_{i,k}^c$, given by
$$\omega_{i,k}^c = \frac{1}{\Delta t} \left( \operatorname{arccos} \left( \frac{ \begin{bmatrix} 1 & 0 \end{bmatrix} \boldsymbol{v}_{i,k}^c}{\|\boldsymbol{v}_{i,k}^c\|} \right)-\theta_{i,k}\right)$$
The desired velocity command $\boldsymbol{v}_{i,k}^c$ is given by
\begin{equation} \label{eq:desirsed_velocity}
    \begin{aligned}
    \boldsymbol{v}_{i,k}^c = \text{sat}(U_{\boldsymbol{v}_{i}^f},\boldsymbol{v}_{i,k}^f) + \boldsymbol{v}_{i,k}^t
    \end{aligned}
\end{equation}
where 
\begin{equation}\label{eq:desirsed_velocity_details}
\left\{\begin{array}{l}
    \boldsymbol{v}_{i,k}^f =  \sum_{j \in \mathcal{N}_i \cup \{o\}} \kappa_{ij} \left( \boldsymbol{p}_{ij,k}^{*} - \hat{\boldsymbol{p}}_{ij,k}^+ \right) \\
     \boldsymbol{v}_{i,k}^t = \text{sat}(U_{\boldsymbol{v}_{o}},\hat{\boldsymbol{v}}_{o,k}^i) + \boldsymbol{v}_{io,k}^*
\end{array}
\right.
\end{equation}
In \eqref{eq:desirsed_velocity_details}, $\kappa_{ij}>0$ represent control gains, and $U_{\boldsymbol{v}_{i}^f} > 0$ is the bound of $\boldsymbol{v}_{i,k}^f$. Specifically, $\boldsymbol{v}_{i,k}^f$ is a consensus-based formation control term that depends on relative positions, while $\boldsymbol{v}_{i,k}^t$ is a feedforward tracking term comprising the estimated target velocity $\hat{\boldsymbol{v}}_{o,k}^i = \boldsymbol{v}_{i,k} - \hat{\boldsymbol{v}}_{io,k}^+$ and the desired relative velocity $\boldsymbol{v}_{io,k}^*$. The estimated target velocity $\hat{\boldsymbol{v}}_{o,k}^+$ is reasonably bounded by $U_{\boldsymbol{v}_{o}}$, consistent with the actual target velocity. Thus, the velocity of each USV is bounded by 
\begin{equation*}
    U_{\boldsymbol{v}_{i}} = U_{\boldsymbol{v}_{i}^f} + U_{\boldsymbol{v}_{o}} + U_{\boldsymbol{v}_{io}^*}
\end{equation*} 
From the form of $U_{\boldsymbol{v}_{i}}$, it can be seen that the velocity of the USV is greater than that of the target, thereby enabling the tracking task. 
In addition, $U_{\boldsymbol{v}_{o}}$ and $U_{\boldsymbol{v}_{i}}$ are determined by the target's maximum speed and the USV's physical capabilities, respectively. In the controller design, $U_{\boldsymbol{v}_{io}^*}$ should be set as large as possible, and $U_{\boldsymbol{v}_{i}^f}$ as relatively small. This parameter tuning ensures that the PE generated by $\boldsymbol{v}_{io,k}^*$  can effectively counterbalance the effects of $\boldsymbol{v}_{i,k}^f$ and $\hat{\boldsymbol{v}}_{o,k}^i$. Consequently, it guarantees that the relative motion induced by the controller satisfies the PE condition, a point that will be further analyzed in Section V.

\subsection{Obstacle Avoidance Control}
In practical applications, USVs must avoid obstacles such as reefs. To this end, each USV $\mathcal{V}_i$ is assumed to be capable of acquiring environmental obstacle information in real-world deployments using sensors such as omnidirectional sonar. This information includes, for example, the contours of reefs or islands. The obstacle boundaries are represented by a set of $M$ discrete points relative to the vehicle $\{ \boldsymbol{p}_{im,k} \}$ for $m = \{1,\ldots,M\}$. It is worth noting that the relative position of each obstacle boundary point is typically not obtained directly. For sensors such as sonar or LiDAR, these positions are derived by combining bearing and range information. A detailed treatment of this derivation is beyond the scope of this paper, and a simplified approach is adopted here. A collision avoidance mechanism is triggered if $\|\boldsymbol{p}_{im,k}\| < d_{s}$, where $d_{s}$ is the safety distance.

The collision-free control term based on the vector field \cite{zhouAgileCoordinationAssistive2018,Feng2025IRA} is given by
    \begin{equation}\label{eq:collision-free}
        \boldsymbol{v}_{i,k}^{s} = \sum_{m=1}^M k_1 \exp \left(-\frac{\left\lVert \boldsymbol{p}_{im,k} \right\rVert^2 }{2k_2 d_{s}^2}\right) \frac{\boldsymbol{p}_{im,k}}{\left\lVert \boldsymbol{p}_{im,k} \right\rVert}
    \end{equation}
where $k_1,k_2\in \mathbb{R}$ are positive scalar gains, selected according to the safety distance $d_{s}$. To incorporate obstacle avoidance, the safety term $\boldsymbol{v}_{i,k}^{s}$ is directly added to the control law in \eqref{eq:desirsed_velocity}, yielding
\begin{equation} \label{eq:desirsed_velocity_cofree}
    \begin{aligned}
    \boldsymbol{v}_{i,k}^c = \text{sat}(U_{\boldsymbol{v}_{i}^f},\boldsymbol{v}_{i,k}^f) + \boldsymbol{v}_{i,k}^t + \text{sat}(U_{\boldsymbol{v}_{i}^s},\boldsymbol{v}_{i,k}^{s})
    \end{aligned}
\end{equation}
where $U_{\boldsymbol{v}_{i}^s}>0$ is the upper bound for the collision avoidance term. Consequently, the final controller is identical in form to that defined in \eqref{eq:controller}.

\section{\textsc{Observability Analysis}} \label{section_Analysis}
In this section, we present the analysis results on the observability of the MCKF and PLKF under the cooperative circumnavigation control algorithm given in Section \ref{section_control}. The analysis is performed by assuming the coupled oscillator \eqref{eq:osc} has reached its equilibrium after some time step $k_s$. 
In addition, measurement noises are neglected to focus on the theoretical analysis.
It should be emphasized that observability is a prerequisite for the Kalman filtering framework to estimate the state from measurements. The effects of non-Gaussian noise on estimation accuracy and robustness are considered only under the premise that observability is guaranteed.
The following assumptions are further introduced to support our follow-up analysis.

\begin{assumption} \label{assump:colinear}
    The desired yaw rate $\omega_i^*$ is sufficiently small, ensuring that for some integer $N \geq 2$ and for any $l \geq k_s$, the vectors in $\{\boldsymbol{v}_{io,k}^*\}_{k=l}^{l+N-1}$ are pairwise non-collinear; that is, for any two distinct time steps $\{k_1,k_2\} \subseteq \{l,l+1,\ldots,l+N-1\}$, $\boldsymbol{v}_{io,k_1}^*$ and $\boldsymbol{v}_{io,k_2}^*$ are non-collinear.
\end{assumption}

Since the sampling interval $\Delta t$ is typically constrained by sensor capabilities, for a given $\Delta t$, desired circumnavigation radius $r_{io}^*$, and maximum USV speed $U_{\boldsymbol{v}_i}$, Assumption~\ref{assump:colinear} is readily satisfied in practice. Consequently, the USV requires at least $N$ time steps to complete one full circumnavigation of the target. 
Based on Assumption \ref{assump:colinear}, the desired relative velocities $ \boldsymbol{v}_{io,k}^* $ and $ \boldsymbol{v}_{ij,k}^* = \boldsymbol{v}_{io,k}^* - \boldsymbol{v}_{jo,k}^* $ satisfy the PE condition. Formally, there exist $\lambda_{io,2}^* \geq \lambda_{io,1}^* > 0, \lambda_{ij,2}^* \geq \lambda_{ij,1}^* > 0$ such that for all $l \geq k_s $
    \begin{equation}\label{eq:pe_vi}
       \lambda_{io,1}^* \boldsymbol{I} \leqslant \boldsymbol S_{io,l}^* \triangleq \sum_{k=l}^{l+N-1}{\boldsymbol{v}_{io,k}^* ({\boldsymbol{v}_{io,k}^*})^{T} }\leqslant \lambda_{io,2}^* \boldsymbol{I}
    \end{equation}
    \begin{equation}\label{eq:pe_vij}
       \lambda_{ij,1}^* \boldsymbol{I} \leqslant \boldsymbol S_{ij,l}^* \triangleq \sum_{k=l}^{l+N-1}{\boldsymbol{v}_{ij,k}^* ({\boldsymbol{v}_{ij,k}^*})^{T} }\leqslant \lambda_{ij,2}^* \boldsymbol{I}
    \end{equation}
A geometric interpretation of this relation is provided in \cite{liuFormationControlEnclosing2025a}. Here, $\lambda_{io,1}^*$ and $\lambda_{io,2}^*$ (or $\lambda_{ij,1}^*$ and $\lambda_{ij,2}^*$) denote the minimum and maximum eigenvalues of $\boldsymbol S_{io,l}^*$ (or $\boldsymbol S_{ij,l}^*$), respectively. These eigenvalues are determined by the shift phase $\omega_i^* \Delta t$, the specified radius $r_{io}^*$ and the window size $N$. If $N\omega_i^* \Delta t\leq 2\pi$, we have
\begin{equation} \label{eq:eigenvalues}
\left\{
    \begin{aligned}
    &\lambda_{io,1}^*,\lambda_{io,2}^* = N U_{\boldsymbol{v}_{io}^*}^2 \left(\frac{1}{2} \pm \frac{\sin{ [(N-1)\omega_i^* \Delta t} ]}{2(N-1)\omega_i^* \Delta t} \right) \\
    &\lambda_{ij,1}^*,\lambda_{ij,2}^* = N U_{\boldsymbol{v}_{ij}^*}^2 \left(\frac{1}{2} \pm \frac{\sin{ [(N-1)\omega_i^* \Delta t} ]}{2(N-1)\omega_i^* \Delta t} \right)
    \end{aligned}
\right.
\end{equation}
where $U_{\boldsymbol{v}_{io}^*}=\frac{2 r_{io}^* }{\Delta t} \sin{ \left( \frac{\omega_{i}^*\Delta t}{2} \right)}, U_{\boldsymbol{v}_{ij}^*}=\frac{ 2 U_{\boldsymbol{v}_{io}^*} }{\Delta t}\sin{ \left( \frac{\omega_{i}^*\Delta t}{2} \right)}$ are the Euclidean norms of $\boldsymbol{v}_{io,k}^*$ and $\boldsymbol{v}_{ij,k}^*$, respectively.

Based on the preceding analysis, we now demonstrate that under controller \eqref{eq:controller}, the relative velocities between USVs, $\boldsymbol{v}_{ij,k} = \boldsymbol{v}_{i,k} - \boldsymbol{v}_{j,k}$, and those between each USV $\mathcal{V}_i$ and the target, $\boldsymbol{v}_{io,k} = \boldsymbol{v}_{i,k} - \boldsymbol{v}_{o,k}$, also satisfy the PE condition. 

\begin{lemma} \label{lemma:pe_vij}
    If \eqref{eq:pe_vij} holds and appropriate values for $\omega_i^* \Delta t$ and $r_{io}^*$ are chosen such that $\lambda_{ij,1}^* > \sqrt{N} \left(U_{\boldsymbol{v}_{i}^f} + U_{\boldsymbol{v}_{j}^f} + 2U_{\boldsymbol{v}_{o}}\right)$, then the relative velocities $\boldsymbol{v}_{ij,k}$ between USVs satisfies the following PE condition for $l \geq k_s$
    \begin{equation}\label{eq:pe_vij_}
       \lambda_{ij,1} \boldsymbol{I} \leqslant \boldsymbol S_{ij,l} \triangleq \sum_{k=l}^{l+N-1}{\boldsymbol{v}_{ij,k} \boldsymbol{v}_{ij,k}^{T} }\leqslant \lambda_{ij,2} \boldsymbol{I}
    \end{equation}
    where 
    \begin{equation*}
    \lambda_{ij,1} = \left[\lambda_{ij,1}^* - \sqrt{N} \left(U_{\boldsymbol{v}_{i}^f} + U_{\boldsymbol{v}_{j}^f} + 2U_{\boldsymbol{v}_{o}}\right) \right]^2
    \end{equation*} 
    and 
    \begin{equation*}
    \lambda_{ij,2} = \left[\lambda_{ij,2}^* + \sqrt{N} \left(U_{\boldsymbol{v}_{i}^f} + U_{\boldsymbol{v}_{j}^f} + 2U_{\boldsymbol{v}_{o}}\right) \right]^2
    \end{equation*}
\end{lemma}

\begin{IEEEproof}
    See Appendix A.
\end{IEEEproof}

\begin{lemma} \label{lemma:pe_vio}
    If \eqref{eq:pe_vi} holds and appropriate values for $\omega_i^* \Delta t$ and $r_{io}^*$ are chosen such that $\lambda_{io,1}^* > \sqrt{N} \left(U_{\boldsymbol{v}_{i}^f} + 2U_{\boldsymbol{v}_{o}}\right)$, then the relative velocities $\boldsymbol{v}_{io,k}$ between $\mathcal{V}_i$ and the target satisfies the following PE condition for $l \geq k_s$
    \begin{equation}\label{eq:pe_vi_}
       \lambda_{io,1} \boldsymbol{I} \leqslant \boldsymbol S_{io,l} \triangleq \sum_{k=l}^{l+N-1}{\boldsymbol{v}_{io,k} \boldsymbol{v}_{io,k}^{T} }\leqslant \lambda_{io,2} \boldsymbol{I}
    \end{equation}
    where 
     \begin{equation*}
    \lambda_{io,1} = \left[\lambda_{io,1}^* - \sqrt{N} \left(U_{\boldsymbol{v}_{i}^f} + 2U_{\boldsymbol{v}_{o}}\right) \right]^2
    \end{equation*}
    and 
    \begin{equation*}
    \lambda_{io,2} = \left[\lambda_{io,2}^* + \sqrt{N} \left(U_{\boldsymbol{v}_{i}^f} + 2U_{\boldsymbol{v}_{o}}\right) \right]^2
    \end{equation*}
\end{lemma}

\begin{IEEEproof}
    See Appendix B.
\end{IEEEproof}

To satisfy the PE conditions in Lemma~\ref{lemma:pe_vij} and Lemma~\ref{lemma:pe_vio}, the desired radius $r_{io}^*$ should be increased in response to a larger upper bound on the target's speed, $U_{\boldsymbol{v}_o}$. For a fixed $r_{io}^*$, the PE degree can be enhanced by increasing the desired yaw rate $\omega_i^*$ within the USV's maximum speed constraint $U_{\boldsymbol{v}_i}$, which correspondingly elevates the relative velocity bounds $U_{\boldsymbol{v}_{io}^*}$ and $U_{\boldsymbol{v}_{ij}^*}$.
It is important to note that the theoretical analysis provides sufficient, not necessary, conditions for PE. The condition may still be satisfied even if the parameters do not meet the criteria outlined in the lemmas. Nevertheless, this analysis offers clear and practical guidelines for parameter selection. The core principle is to select sufficiently large values for $r_{io}^*$ and $\omega_i^*$ such that the resulting relative speed bound $U_{\boldsymbol{v}_{io}^*}$ is large enough to dominate the combined effects of the USV's pursuit speed $U_{\boldsymbol{v}_{i}^f}$ and the target speed $U_{\boldsymbol{v}_o}$, thereby fulfilling the PE condition. This requirement is substantially relaxed for stationary or uniformly moving targets.

Based on Lemmas \ref{lemma:pe_vij} and \ref{lemma:pe_vio}, we now present the main observability results for the MCKF and PLKF. Measurement noise is neglected in this analysis to focus on the system's structural properties.

\begin{theorem} \label{theorem:MCKF_ob}
    Under the controller \eqref{eq:controller} and Assumption \ref{assump:colinear}, the relative positions $\boldsymbol{p}_{ij,k}$ between $\mathcal{V}_i$ and $\mathcal{V}_j$ are observable in the MCKF system, which incorporates the process model \eqref{eq:propogation_ij} and the measurement model \eqref{eq:rel_meas}. Specifically, the Gramian matrix for $l \geq k_s$
    \begin{equation}\label{eq:Gramian_MCKF}
       \boldsymbol \Phi_{ij,l} \triangleq \sum_{k=l}^{l+N-1}{\boldsymbol{\delta}_{ij,k} \boldsymbol{\delta}_{ij,k}^{T} } > 0
    \end{equation}    
    is positive definite.
\end{theorem} 
\begin{IEEEproof}
    See Appendix C.
\end{IEEEproof}

To further analyze the relative state estimation between the USV and the target, we first present the following assumptions.
\begin{assumption} \label{assump:maneuverability}
    It is assumed that the distance between each USV and the target satisfies $d_{s} \leq \| \boldsymbol{p}_{io,k} \|  \leq U_{\boldsymbol{p}_{io}}$. Furthermore, parameters $\omega_i^*, \Delta t$ and $r_{io}^*$ can be selected such that $ \lambda_{io,1} > \Delta t \sqrt{N} U_{\boldsymbol{p}_{io}} $
\end{assumption}

Given the finite perception range of onboard sensors, the distance between a USV and its target is reasonably assumed to be bounded above by $U_{\boldsymbol{p}_{io}}$. This bound typically decreases as the USV approaches the target during a tracking mission. Furthermore, as indicated by \eqref{eq:eigenvalues}, $\lambda_{io,1}$ is approximately proportional to $(r_{io}^*)^4$. This relationship suggests that system performance can be enhanced by initializing the task with a larger desired circumnavigation radius $r_{io}^*$. In addition, the desired yaw rate $\omega_i^*$ should be selected as large as possible and the sampling frequency should be increased (\emph{i.e.}, $\Delta t$ reduced), subject to the system's physical constraints. Notably, this parameter selection criterion is consistent with the one obtained from Lemma~\ref{lemma:pe_vij} and Lemma~\ref{lemma:pe_vio}.

\begin{theorem} \label{theorem:PLKF_ob}
    Under the controller \eqref{eq:controller} and subject to Assumption \ref{assump:maneuverability} and Lemma \ref{lemma:pe_vio}, the relative positions $\boldsymbol{p}_{io,k}$ between $\mathcal{V}_i$ and the unknown maneuvering target are observable in the PLKF system, which integrates the process model \eqref{eq:propogation_oi} and the measurement model \eqref{eq:bearing_PL}. Specifically, the Gramian matrix for $l \geq k_s$
    \begin{equation}\label{eq:Gramian_PLKF}
       \boldsymbol \Phi_{io,l} \triangleq \sum_{k=l}^{l+N-1}  \boldsymbol{G}_{{\boldsymbol{b}}_{io},k}^T \boldsymbol{G}_{{\boldsymbol{b}}_{io},k} > 0
    \end{equation}    
    is positive definite, where $ \boldsymbol{G}_{{\boldsymbol{b}}_{io},k} = \boldsymbol{I}  - \boldsymbol{b}_{io,k}\boldsymbol{b}_{io,k}^T $ without considering noises. 
\end{theorem} 

\begin{IEEEproof}
    See Appendix D.
\end{IEEEproof}

The analysis shows that to ensure observability, the USV must possess sufficiently high maneuverability to overcome the interference caused by unknown target maneuvers, thereby guaranteeing sufficiently rich relative motion. A direct corollary is that if the target is stationary or moving at constant velocity, its velocity also becomes observable. In practice, for targets with unknown maneuvers or when considering the influence of real-world factors such as water currents, the state estimation of the target can only be confined within a certain range. Nevertheless, this is sufficient to meet the requirements of most rescue and similar missions.

\section{\textsc{Simulation}} \label{section_sim}

This section presents simulation studies designed to comprehensively validate the effectiveness and robustness of the proposed relative localization and cooperative circumnavigation framework. 
The performance of the inter-USV relative positioning algorithm is first verified under non-Gaussian noise conditions, followed by a validation of the overall framework. All simulations are conducted on a personal computer equipped with an Intel Core i7-10850H CPU clocked at 2.70GHz and 32GB of RAM.

\renewcommand{\arraystretch}{1.3}
\begin{table}[htbp]
\centering
\caption{PARAMETERS IN SIMULATION}
\label{tab:parameters}
\begin{tabular}{@{}ll@{}}
\toprule
\textbf{Parameter} & \textbf{Value}  \\ \midrule
Sampling period $\Delta t$ & $0.1$s\\
Desired yaw rate $\omega_i^*$ & $\pi/30$rad/s \\
Safety distance $d_s$ & $3.0$m \\
Physical upper bounds $U_{\omega_i^c}, U_{\boldsymbol{v}_{i}}, U_{\boldsymbol{v}_{o}}$ & $\pi/3,5,0.5$\\
Controller upper bounds $U_{\boldsymbol{v}_{i}^f}, U_{\boldsymbol{v}_{i}^s}$ & $1.5,0.7$\\
Formation gains $\kappa_{ij}, \kappa_{io}$ & $1.5, 0.9$ \\
Collision avoidance gains $k_1, k_2$ & $1.7, 0.5$ \\
Target process noise $\boldsymbol{Q}_o$ & $0.01\boldsymbol{I}$ \\
Bearing measurement noise $\boldsymbol{R}_{\boldsymbol{b}_{io}}$ & $10^{-4}\boldsymbol{I}$ \\
Displacement measurement noise $\boldsymbol{R}_{\boldsymbol\delta_{i}}$ & $10^{-6}I$ \\
MCKF kernel bandwidths $[\sigma_1, \sigma_2, \sigma_3]$ & $[9999, 9999, 0.8]$ \\
MCKF stopping threshold $\iota$ & $10^{-6}$\\
MCKF maximum iterations $L_{\max}$ & $10$\\
Initial covariance $\boldsymbol{P}_{ij,0}^+$ & $25\boldsymbol{I}$ \\
Initial covariance $\boldsymbol{P}_{io,0}^+$ & $\boldsymbol{I}$ \\ \bottomrule
\end{tabular}
\end{table}

The principal simulation parameters are summarized in Table \ref{tab:parameters}. The initial estimates of the inter-USV relative positions are generated as $\hat{\boldsymbol{p}}_{ij,0}^+ \sim \mathcal{N}(\boldsymbol{p}_{ij,0}, \boldsymbol{P}_{ij,0}^+)$. Likewise, the initial estimate of the relative state $\boldsymbol{x}_{io}$ between USV and the target is randomly drawn from $\hat{\boldsymbol{x}}_{io,0}^+ \sim \mathcal{N}(\boldsymbol{x}_{io,0}, \boldsymbol{P}_{io,0}^+)$. In the implementation of the MCKF, the process noise follows a Gaussian distribution. Accordingly, the kernel bandwidth parameters $\sigma_1$ and $ \sigma_2 $ are set to sufficiently large values such that the MCKF behaves similarly to the standard Kalman filter. In contrast, to address the heavy-tailed characteristics of the equivalent measurement noise, the corresponding kernel bandwidth $\sigma_3$ is set to a small value, thereby enhancing the filter's robustness against outliers. In addition, the required number of iterations for MCKF is generally small and related to $\iota$ \cite{chenMaximumCorrentropyKalman2017}. Typically, the smaller the $\iota$, the larger the number of iterations. In the simulations of this paper, the maximum number of iterations does not exceed 4, and conservatively, $L_{\max}=10$ is chosen.

The following performance metrics are introduced to quantify the estimation and tracking performance.

1) Relative position estimation error $e^p_{ij,k}$ among
USVs are introduced to quantify the relative localization accurracy, which are defined as
\begin{equation*}
   e^p_{ij,k}= \max\limits_{i\in\mathcal{V}_i} \left( \max\limits_{j\in\mathcal{N}_i} \|{\hat{\boldsymbol{p}}_{ij,k}^+ - \boldsymbol{p}_{ij,k}\|} \right) 
\end{equation*}
where a smaller value of $e^p_{ij,k}$ implies better relative localization performance.

2) Relative position estimation error $e^p_{io,k}$ and relative velocity estimation error $e^v_{io,k}$ between the USV and the target are introduced to quantify the target state estimation performance.
\begin{equation*}
\left\{\begin{array}{l}
e^p_{io,k} = \max\limits_{i\in\mathcal{V}_i} \|\hat{\boldsymbol{p}}_{io,k}^+ - \boldsymbol{p}_{io,k}\|
\\
e^v_{io,k} = \max\limits_{i\in\mathcal{V}_i} \|\hat{\boldsymbol{v}}_{io,k}^+ - \boldsymbol{v}_{io,k}\|
\end{array}
\right.
\end{equation*}
where smaller values of $e^p_{io,k}$ and $e^v_{io,k}$ indicate better state estimation performance.

3) Tracking error $e_{c,k}$ between the formation center trajectory and the target position are employed to characterize the overall position tracking performance of the cooperative circumnavigation controller, which is defined as 
\begin{equation*}
e_{c,k} = \left\|\frac{1}{n}\sum_{i\in\mathcal{V}_i} {\boldsymbol{p}}_{i,k} - \boldsymbol{p}_{o,k}\right\|
\end{equation*}
where a smaller value of $e_{c,k}$ shows better tracking performance.

Additionally, the instant at which all USVs have first converged to the desired circular formation such that their radii lie within $\pm 10\%$ of the desired radius, and thereafter remain within this bound for the duration of the simulation is referred to as the settling time. Furthermore, the average execution time per iteration, denoted as Avg. Run Time, is used to validate the real-time feasibility and computational efficiency of the proposed framework.

\subsection{Comparison of Inter-USV Relative Localization}
The performance of the proposed MCKF algorithm is evaluated against several baseline estimators, including standard methods such as Recursive Least Squares (RLS) \cite{nguyen2020} and the classical Kalman Filter (KF), as well as robust approaches including the Adaptive Kalman Filter (AKF) \cite{Huang2018Adaptive} and Robust Least Squares with a Huber loss function (Huber-RLS) \cite{tong2023functional}. While AKF targets a different robustness aspect than the heavy-tailed measurement noise considered in this paper, it is included as a representative adaptive robust filter to provide a broader comparison landscape. In contrast, Huber-RLS directly relevant to heavy-tailed noise and serves as a direct robust comparator. To assess the effectiveness of the proposed algorithm, we consider two distinct noise models in simulation. In addition, comparative experiments are conducted using a real-world UWB dataset.

\textit{1) Performance under different range measurement noise models for inter-USV relative localization:} We compare the relative position estimation results obtained by the aforementioned estimators under two different range measurement noise models: the Gaussian mixture model \eqref{eq:r_noise_model_multi_Gaus} and the log-normal distribution model \eqref{eq:r_noise_model_logGaus}.
For the Gaussian mixture model \eqref{eq:r_noise_model_multi_Gaus}, we set $\epsilon=0.2, \boldsymbol{R}_{{\mu}_{r_{ij}}}=0.4, \boldsymbol{R}_{{\mu}_{r_{ij}}}^{outliers}=10$; for the log-normal distribution model \eqref{eq:r_noise_model_logGaus}, we set $\boldsymbol{R}_{{\mu}_{r}}=0.7, r_s = 0.75$. To isolate the effect of noise characteristics on inter-USV relative localization, we consider a scenario in which two USVs perform circular motion along desired trajectories with $r_{io}^*=5$m and $\omega_i^*=\pi/30$rad/s, in the absence of a target.

\begin{figure}[htbp]
    \centering
    \includegraphics[width=0.48\textwidth]{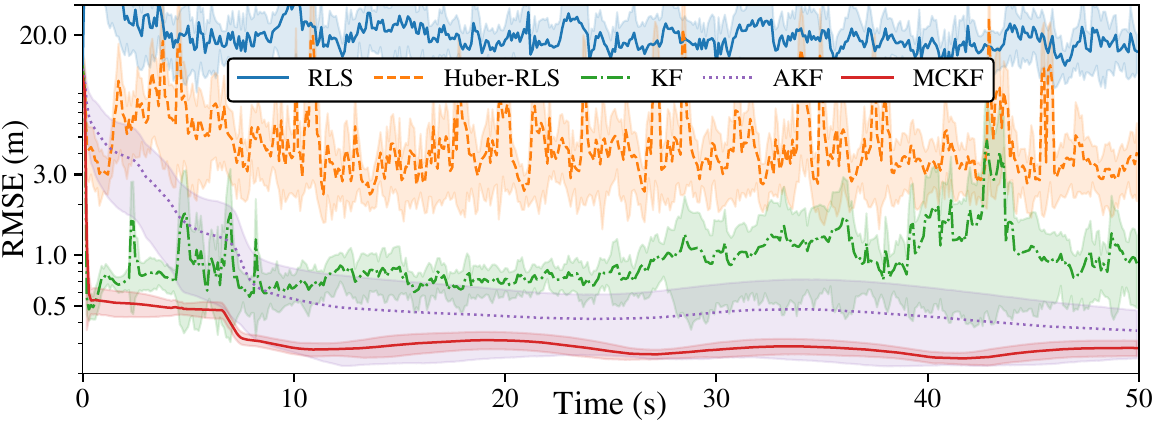}
    \caption{RMSE of relative position estimation under Gaussian mixture model. The shaded regions represent $\pm$1 standard deviation across Monte Carlo runs. The vertical axis is in logarithmic coordinates.}
    \label{fig:rmse_mixgaus}
\end{figure}

\begin{figure}[htbp]
    \centering
    \includegraphics[width=0.48\textwidth]{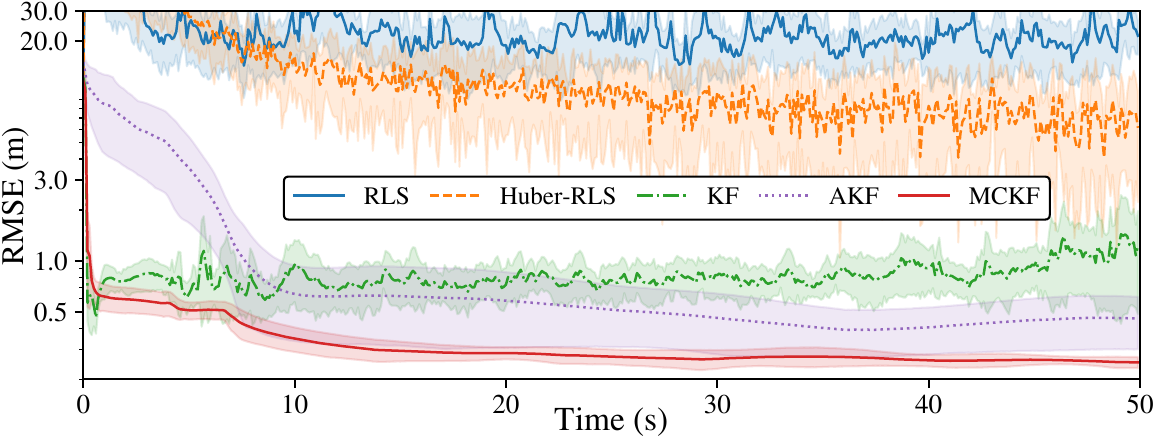}
    \caption{RMSE of relative position estimation under log-normal distribution model. The shaded regions represent $\pm$1 standard deviation across Monte Carlo runs. The vertical axis is in logarithmic coordinates.}
    \label{fig:rmse_loggaus}
\end{figure}

\begin{figure}[htbp]
    \centering
    \includegraphics[width=0.48\textwidth]{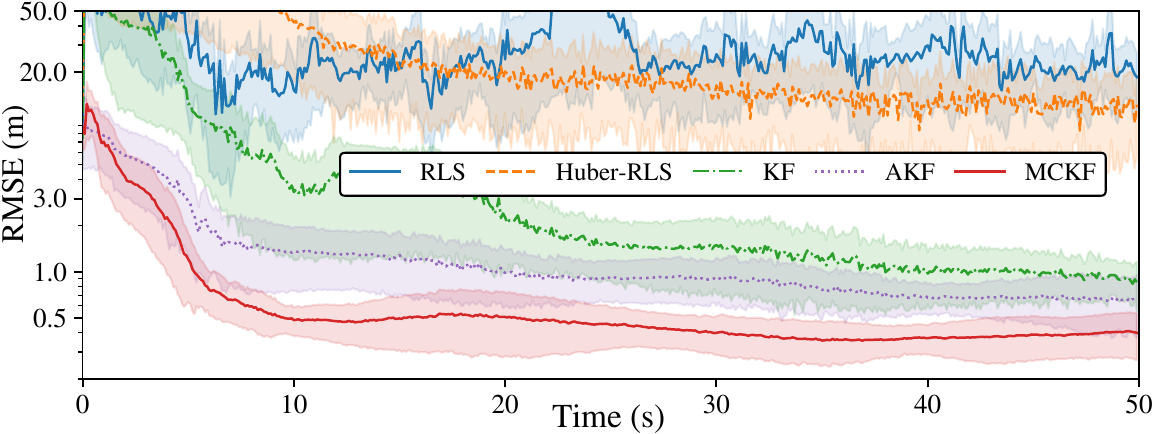}
    \caption{RMSE of relative position estimation with real-world UWB dataset. The shaded regions represent $\pm$1 standard deviation across Monte Carlo runs. The vertical axis is in logarithmic coordinates.}
    \label{fig:UWBnoise}
\end{figure}

A total of 100 Monte Carlo simulations are conducted, and the estimation accuracy is quantified using the root-mean-square error (RMSE). The RMSE results of the different relative localization algorithms are shown in Fig.~\ref{fig:rmse_mixgaus} and Fig.~\ref{fig:rmse_loggaus} for the Gaussian mixture model and the log-normal distribution model, respectively. A comparison of the two figures reveals that, despite the differences in the underlying range measurement noise models, the relative performance trends among the estimators remain consistent. The underlying reason is that, under both noise models, the resulting equivalent measurement noise $\mu_{z_{ij},k}$ exhibits similar statistical characteristics—specifically, a symmetric heavy-tailed distribution with a sharp peak near zero.
Notably, the RLS-based estimator yields a significantly higher RMSE compared to those based on the Kalman filter framework, accompanied by larger variance and less stable estimation results. In contrast, the KF-based algorithms produce smoother estimates. Among them, the proposed MCKF achieves a lower RMSE and a narrower standard deviation range than both the standard KF and the AKF. Overall, the MCKF demonstrates superior estimation performance in handling heavy-tailed UWB noise.

\textit{2) Validation using real-world UWB dataset:} To further evaluate the performance of the proposed algorithm with real sensor data, we conduct comparative experiments using the UTIL dataset \cite{zhaoUTILUltrawidebandTimedifferenceofarrival2024}, a real-world UWB measurement dataset. It should be noted that this dataset involves measurements between one dynamic agent and stationary anchors. Therefore, it reflects only a single pairwise ranging scenario and does not constitute a true multi-agent validation involving multiple moving USVs. Based on this dataset, we assume a scenario of relative position estimation between one stationary USV and one moving USV.

The experimental results are presented in Fig.~\ref{fig:UWBnoise}. It can be observed that the RLS-based algorithm exhibits larger fluctuations and higher estimation errors, whereas the methods based on the Kalman filter framework produce noticeably smoother estimates. Among them, the proposed MCKF achieves faster convergence and yields smaller estimation errors. Overall, the MCKF maintains superior performance when applied to real-world data. These results further demonstrate the practical advantages and application potential of the MCKF in real UWB sensor scenarios.

\subsection{Effectiveness of the Overall Framework}

In this subsection, a comprehensive evaluation of the proposed framework's effectiveness under various operational conditions is conducted. The investigation begins by comparing the formation tracking performance when different relative position estimators are employed. Subsequently, the influence of the two considered noise models on positioning and tracking accuracy is analyzed. The scalability of the approach is then verified by configuring the formation with different radii and fleet sizes. To further assess robustness in challenging environments, scenarios involving a weakly maneuvering target and static obstacles are simulated. Finally, the framework's resilience is tested under realistic communication constraints, including message dropouts and the dynamic addition or removal of USVs during mission execution.

\textit{1) Performance under different relative position estimators:} Consider a scenario in which four USVs track a target moving at a constant velocity of $\boldsymbol{v}_{o} = [0.5, 0]^{T}$m/s. The four USVs are initially located at $(-6, -7)$m, $(6, 6)$m, $(-10, 3)$m, and $(-5, 6)$m, respectively, with the target initially positioned at $(1, 1)$m. The initial matrices are set to $\boldsymbol{P}_{ij,0}^+ = 0.04\boldsymbol{I}$ and $\boldsymbol{P}_{io,0}^+ = \boldsymbol{I}$. The desired circumnavigation radius and yaw rate are set to $r_{io}^*=5$m and $\omega_i^*=\pi/30$rad/s, respectively. The range measurment noise is modeled as a Gaussian mixture distribution ${\mu}_{r_{ij},k} \sim 0.8\mathcal{N}(0, 0.4) + 0.2\mathcal{N}(0,10)$.

\begin{figure}[htbp]
    \centering
    \includegraphics[width=0.45\textwidth]{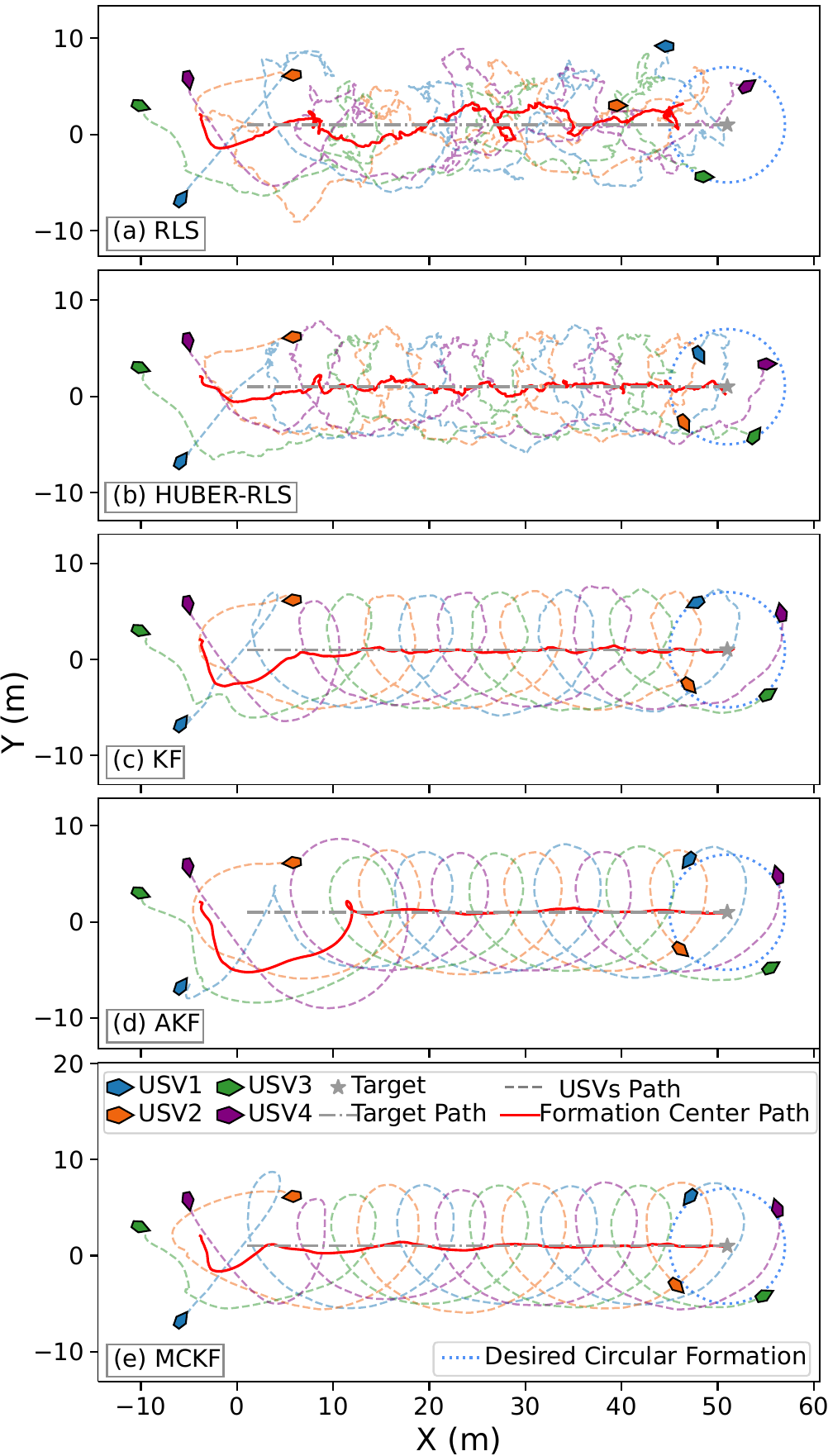}
    \caption{USV target tracking trajectories under different inter-USV relative positioning algorithms.}
    \label{fig:noise1usv4}
\end{figure}

The tracking trajectories obtained using different inter-USV relative positioning algorithms are shown in Fig.~\ref{fig:noise1usv4}. The trajectory of the RLS-based algorithm exhibits significant fluctuations, whereas those of the KF-based algorithms are notably smoother. This observation is consistent with the tracking error curves presented in Fig.~\ref{fig:noise1usv4error}.
It should be noted that $e_{c}$ measures the deviation of the formation center from the target position. Hence, $e_{c}$ can be small even with significant individual deviations, provided these deviations are symmetrically distributed around the formation center, thereby preserving the location of the center itself. As shown in Table~\ref{tab:algorithm_performance}, although the MCKF achieves a smaller estimation error compared to the AKF, its $e_{c}$ is slightly larger. Nevertheless, this difference is marginal and remains within acceptable bounds for practical mission requirements. 
Moreover, as can be observed from the trajectories in Fig.~\ref{fig:noise1usv4} and the quantitative results in Table~\ref{tab:algorithm_performance}, the tracking error of the MCKF-based framework converges more rapidly than that of the AKF-based approach. Furthermore, owing to the iterative computation required for the MCC criterion to converge, the MCKF incurs a longer runtime per iteration. Nonetheless, this runtime remains substantially shorter than the sampling interval, rendering the difference negligible in practice. 
A comprehensive evaluation indicates that the MCKF offers superior estimation accuracy and faster tracking convergence when applied to real sensor data characterized by non-Gaussian noise. Although its average runtime is slightly higher, it remains within an acceptable range. 
Additionally, it is observed that the RLS-based method yields a relatively large $e_{io}^p$ error (see Table~\ref{tab:algorithm_performance} and Fig.~\ref{fig:noise1usv4error}). This may be attributed to excessive errors in inter-USV relative position estimation, which can lead to erratic control inputs and consequently disordered USV trajectories. Such degradation in formation reduces the observability of the PLKF algorithm, thereby resulting in larger $e_{io}^p$ error. This highlights the critical impact of inter-USV relative localization accuracy on overall system performance.

\begin{figure}[htbp]
    \centering
    \includegraphics[width=0.45\textwidth]{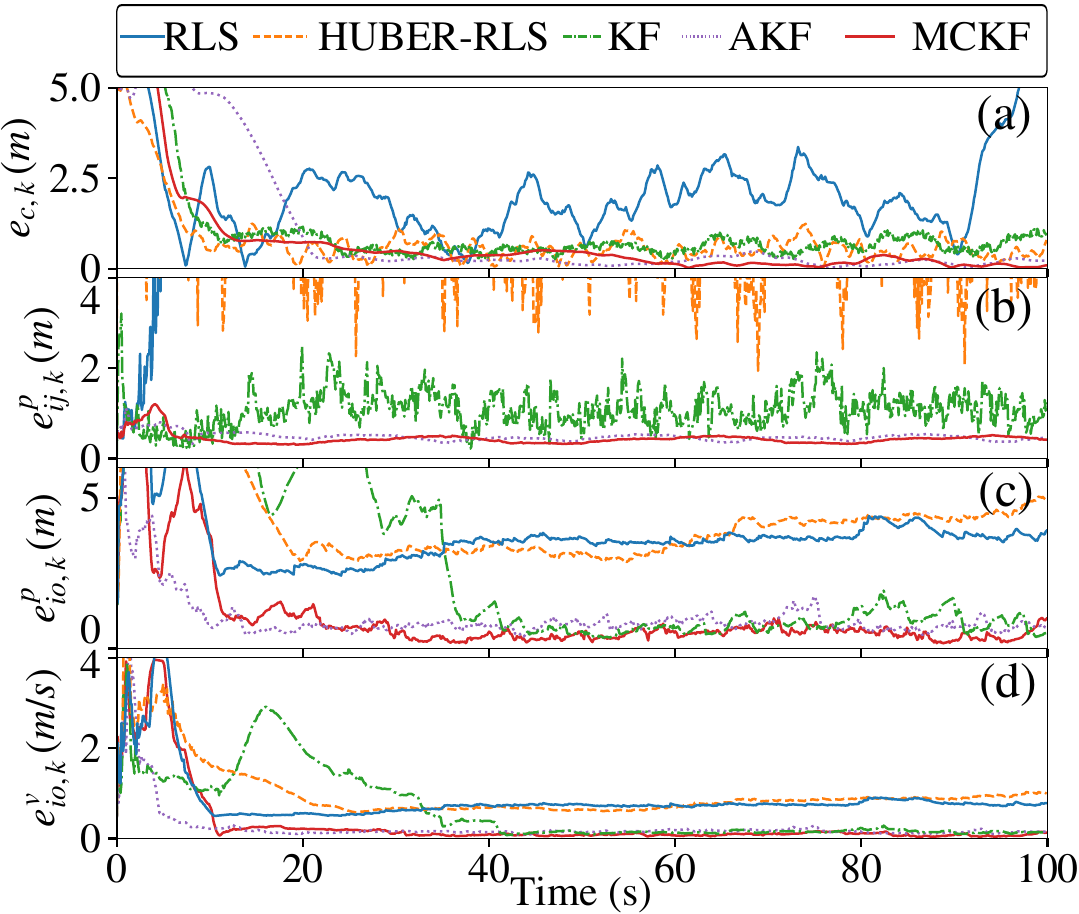}
    \caption{Error evolution for the constant-velocity target tracking scenario under different inter-USV relative positioning algorithms.}
    \label{fig:noise1usv4error}
\end{figure}

\begin{table}[htbp]
\centering
\arrayrulecolor{black} 
\caption{\textsc{Evaluation of Overall Framework Performance with Different Inter-USV Relative Position Estimators}}
\label{tab:algorithm_performance}
\renewcommand{\arraystretch}{1.3}
\setlength{\tabcolsep}{3pt} 
\resizebox{\columnwidth}{!}{ 
\begin{tabular}{lccccc}
\toprule
\textbf{\shortstack{ Estimators}} & \textbf{\shortstack{Mean\\ $e_{ij}^p$ (m)}} & \textbf{\shortstack{Mean\\ $e_{io}^p$ (m)}} & \textbf{\shortstack{Mean\\ $e_{c}$ (m)}} & \textbf{\shortstack{Settling\\ Time (s)}} & \textbf{\shortstack{Avg. Run\\ Time (ms)}} \\ \midrule
RLS       & 33.941 & 4.461 & 3.026 & 12.70 & 0.84 \\
Huber-RLS & 5.685  & 3.928 & 0.499 & 8.80  & 1.02 \\
KF        & 1.160  & 1.101 & 0.485 & 9.60  & 0.95 \\
AKF       & 0.467  & 0.717 & 0.179 & 19.70 & 1.01 \\
MCKF      & 0.449  & 0.476 & 0.196 & 12.30 & 1.50 \\ 
\bottomrule
\end{tabular}
}
\end{table}

\textit{2) Performance under different range measurement noise models:} To further investigate the impact of the MCKF on overall system performance under different range measurement noise characteristics, simulation tests are conducted using both the Gaussian mixture model \eqref{eq:r_noise_model_multi_Gaus} and the log-normal distribution model \eqref{eq:r_noise_model_logGaus}. For the Gaussian mixture model \eqref{eq:r_noise_model_multi_Gaus}, the parameters are set as $\epsilon=0.1$ and $\boldsymbol{R}_{{\mu}_{r_{ij}}}^{outliers}=10$, while $\boldsymbol{R}_{{\mu}_{r_{ij}}}$ is varied among $0.2$, $0.5$, and $0.7$; For the log-normal distribution model \eqref{eq:r_noise_model_logGaus}, we set $r_s=\boldsymbol{R}_{{\mu}_{r}}$, and select $\boldsymbol{R}_{{\mu}_{r}}=0.2,0.5,0.7$. The corresponding simulation results are summarized in Table~\ref{tab:noise_performance}.

As shown in Table~\ref{tab:noise_performance}, the overall system performance remains relatively stable across different noise models, indicating that the MCKF can effectively accommodate variations in the range measurement noise ${\mu}_{r_{ij}}$. The underlying reason is that the MCKF operates on equivalent measurement values derived from range measurements. From \eqref{eq:noise_approximation}, it follows that the equivalent measurement noise $\mu_{ij}$ exhibits consistent statistical characteristics across different noise models for ${\mu}_{r_{ij}}$, specifically a symmetric heavy-tailed distribution with a sharp peak near zero. Furthermore, as the noise intensity increases, the estimation errors gradually increase, while the settling time of the tracking error remains generally stable within a certain range.

\begin{table}[htbp]
\centering
\arrayrulecolor{black} 
\caption{\textsc{Evaluation of Overall Framework Performance with Different Range Measurement Noise Models}}
\label{tab:noise_performance}
\renewcommand{\arraystretch}{1.3}
\setlength{\tabcolsep}{3pt} 
\resizebox{\columnwidth}{!}{ 
\begin{tabular}{lccccc}
\toprule
\textbf{\shortstack{Noise ${\mu}_{r_{ij}}$ }} & \textbf{\shortstack{Mean\\ $e_{ij}^p$ (m)}} & \textbf{\shortstack{Mean\\ $e_{io}^p$ (m)}} & \textbf{\shortstack{Mean\\ $e_{c}$ (m)}} & \textbf{\shortstack{Settling\\ Time (s)}} & \textbf{\shortstack{Avg. Run\\ Time (ms)}} \\ \midrule
\shortstack[l]{$\boldsymbol{R}_{{\mu}_{r_{ij}}}=0.2$} & 0.692 & 0.629 & 0.089 & 11.66 & 1.84 \\
\shortstack[l]{$\boldsymbol{R}_{{\mu}_{r_{ij}}}=0.5$} & 0.696 & 0.658 & 0.115 & 10.98 & 1.88 \\
\shortstack[l]{$\boldsymbol{R}_{{\mu}_{r_{ij}}}=0.7$} & 0.699 & 0.625 & 0.141 & 11.40 & 1.76 \\
\shortstack[l]{$\boldsymbol{R}_{{\mu}_{r}}=0.2$}  & 0.688 & 0.619 & 0.116 & 10.84 & 1.69 \\
\shortstack[l]{$\boldsymbol{R}_{{\mu}_{r}}=0.5$}  & 0.689 & 0.624 & 0.098 & 12.64 & 1.67 \\
\shortstack[l]{$\boldsymbol{R}_{{\mu}_{r}}=0.7$}  & 0.697 & 0.638 & 0.109 & 10.48 & 1.68 \\ 
\bottomrule
\end{tabular}
}
\end{table}

\textit{3) Performance under different numbers of USVs and desired radii:} 
This experiment evaluates the scalability of the proposed framework by varying both the number of USVs and the desired circumnavigation radius, and examines their impact on overall system performance. The range measurement noise is fixed as ${\mu}_{r_{ij},k} \sim 0.8\mathcal{N}(0, 0.4) + 0.2\mathcal{N}(0,10)$, and the desired yaw rate is set to $\omega_i^*=\pi/30$rad/s. To ensure that the USVs can successfully accomplish the circumnavigation task under a larger radius, the physical velocity bound of each USV is increased to $U_{\boldsymbol{v}_{i}}=10$, and the upper bound of the formation control term constraint is raised to $U_{\boldsymbol{v}_{i}^f}=3$.

As shown in Table~\ref{tab:agent_radius_performance}, the tracking task is successfully accomplished under all tested configurations of radii and USV fleet sizes. The algorithm runtime increases with the number of USVs, as expected. Under the same radius, variations in the number of USVs do not lead to significant changes in either the relative position estimation error or the tracking error.
However, owing to the increased values of $U_{\boldsymbol{v}_{i}}$ and $U_{\boldsymbol{v}_{i}^f}$, more pronounced motion variations occur during the convergence phase when the radius is relatively small. Consequently, the settling time becomes longer in such cases. Nevertheless, these more aggressive motion variations simultaneously enhance the PE level of the relative motion among USVs. As a result, the inter-USV relative position estimation error $e_{ij}^p$  is smaller compared to the results obtained with the MCKF in Table~\ref{tab:algorithm_performance}. 
Furthermore, for a fixed number of USVs, an increase in the desired radius leads to a larger $U_{\boldsymbol{v}_{io}^*}$, since $\omega_i^*$ remains constant. This enhances the PE level of the relative motion between each USV and the target. Consequently, as shown in Table~\ref{tab:agent_radius_performance}, the target state estimation error $e_{io}^p$ decreases with increasing radius. These observations are consistent with the theoretical analysis presented in Section V regarding the influence of upper bound settings and the values of $\omega_i^*$ and $r_{io}^*$ on estimation performance.

\begin{table}[htbp]
\centering
\arrayrulecolor{black} 
\caption{\textsc{Evaluation of Overall Framework Performance with Different USV Fleet Sizes and Circumnavigation Radii}}
\label{tab:agent_radius_performance}
\renewcommand{\arraystretch}{1.3}
\setlength{\tabcolsep}{3pt} 
\resizebox{\columnwidth}{!}{ 
\begin{tabular}{lcccccc}
\toprule
\textbf{\shortstack{USV \\ Num.}} & \textbf{\shortstack{Radius\\$r_{io}^*$ (m)}} & \textbf{\shortstack{Mean\\ $e_{ij}^p$ (m)}} & \textbf{\shortstack{Mean\\ $e_{io}^p$ (m)}} & \textbf{\shortstack{Mean\\ $e_{c}$ (m)}} & \textbf{\shortstack{Settling\\ Time (s)}} & \textbf{\shortstack{Avg. Run\\ Time (ms)}} \\ \midrule
n=4  & 5  & 0.245 & 2.855 & 0.312 & 11.43 & 1.69  \\
n=4  & 10 & 0.369 & 0.989 & 0.092 & 11.60 & 1.68  \\
n=4  & 20 & 0.694 & 0.469 & 0.084 & 12.79 & 1.66  \\
n=8  & 5  & 0.226 & 2.658 & 0.188 & 3.57  & 3.24 \\
n=8  & 10 & 0.336 & 0.943 & 0.054 & 5.48  & 3.26 \\
n=8  & 20 & 0.621 & 0.399 & 0.059 & 6.60  & 3.28 \\
n=12 & 5  & 0.224 & 2.634 & 0.147 & 3.31  & 5.10 \\
n=12 & 10 & 0.328 & 0.900 & 0.056 & 5.70  & 4.81 \\
n=12 & 20 & 0.602 & 0.368 & 0.051 & 7.19  & 5.22 \\ 
\bottomrule
\end{tabular}
}
\end{table}

\begin{figure}[htbp]
    \centering
    \includegraphics[width=0.47\textwidth]{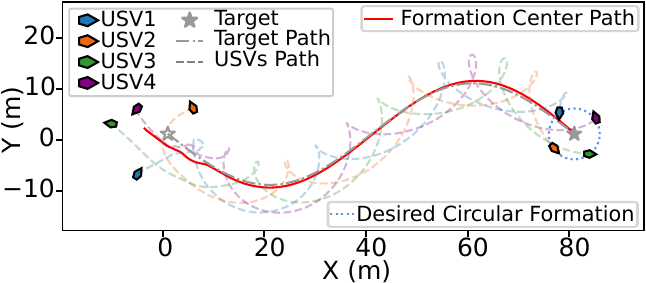}
    \caption{Multi-USV circumnavigation a varying-velocity target in an obstacle-free environment.}
    \label{fig:ARWITHOUTCOBS}
\end{figure}

\begin{figure}[htbp]
    \centering
    \includegraphics[width=0.47\textwidth]{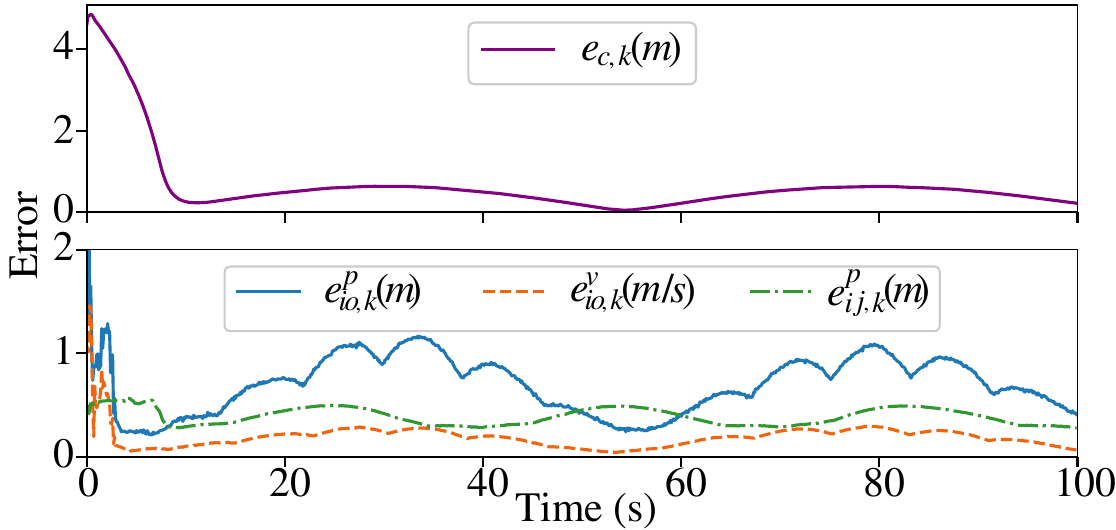}
    \caption{Error evolution for the varying-velocity target tracking scenario.}
    \label{fig:arcobsta}
\end{figure}

\textit{4) Perfomance for tracking maneuvering targets:}
We further evaluate the framework's capability in tracking a maneuvering target. In this scenario, the target moves at a constant speed of $0.5$m/s along a sinusoidal trajectory, which introduces continuous variations in both heading and lateral acceleration. Such motion dynamics pose a significant challenge for maintaining a stable formation center, as the target's varying acceleration demands timely adjustments from the tracking controllers. Fig.~\ref{fig:ARWITHOUTCOBS} illustrates the actual trajectories of the USVs. It can be observed that each USV maintains an approximately circular formation relative to the target even during acceleration and turning maneuvers, thereby confirming the framework's effectiveness in tracking maneuvering targets. The error curves presented in Fig.~\ref{fig:arcobsta} reveal that, although slight fluctuations occur during the target's acceleration phases, all evaluation metrics converge to stable values within a finite time. Specifically, the tracking error of the formation center, as well as the estimation errors for the target's position and velocity, remain within narrow and bounded ranges throughout the entire mission, despite the continuously varying motion profile of the target.

\textit{5) Performance in obstacle environment with occlusions:}
We further evaluate the framework's performance in a complex environment containing two island-shaped obstacles. The simulation results are presented in Fig.~\ref{fig:curveTrajandsta} (a), where the red dashed lines indicate periods of line-of-sight occlusion, i.e., temporary unavailability of bearing measurements between USVs and the target. During such intervals, the PLKF reverts to a prediction-only step as defined in \eqref{eq:predic_only} to prevent the incorporation of unreliable visual observations.
It can be observed that, despite the non-line-of-sight conditions, each USV maintains the formation structure by relying on relative positioning information exchanged with neighbors. Upon restoration of visual measurements, the USVs successfully reacquire the target and resume continuous encircling. As shown in Fig.~\ref{fig:curveTrajandsta} (b), the peak in $e_{c,k}$ at approximately $t = 20$s corresponds to a deliberate deviation, during which the USVs prioritize maintaining the safety distance $d_s$ from obstacle 1. Due to intermittent measurement losses caused by occlusion and the evasive maneuvers required for obstacle avoidance, the error curves exhibit slight fluctuations. Nevertheless, the formation maintains effective encirclement of the target throughout the mission, thereby verifying the system's robustness and safety-critical performance under complex environmental conditions.

\begin{figure}[htbp]
    \centering
    \includegraphics[width=0.47\textwidth]{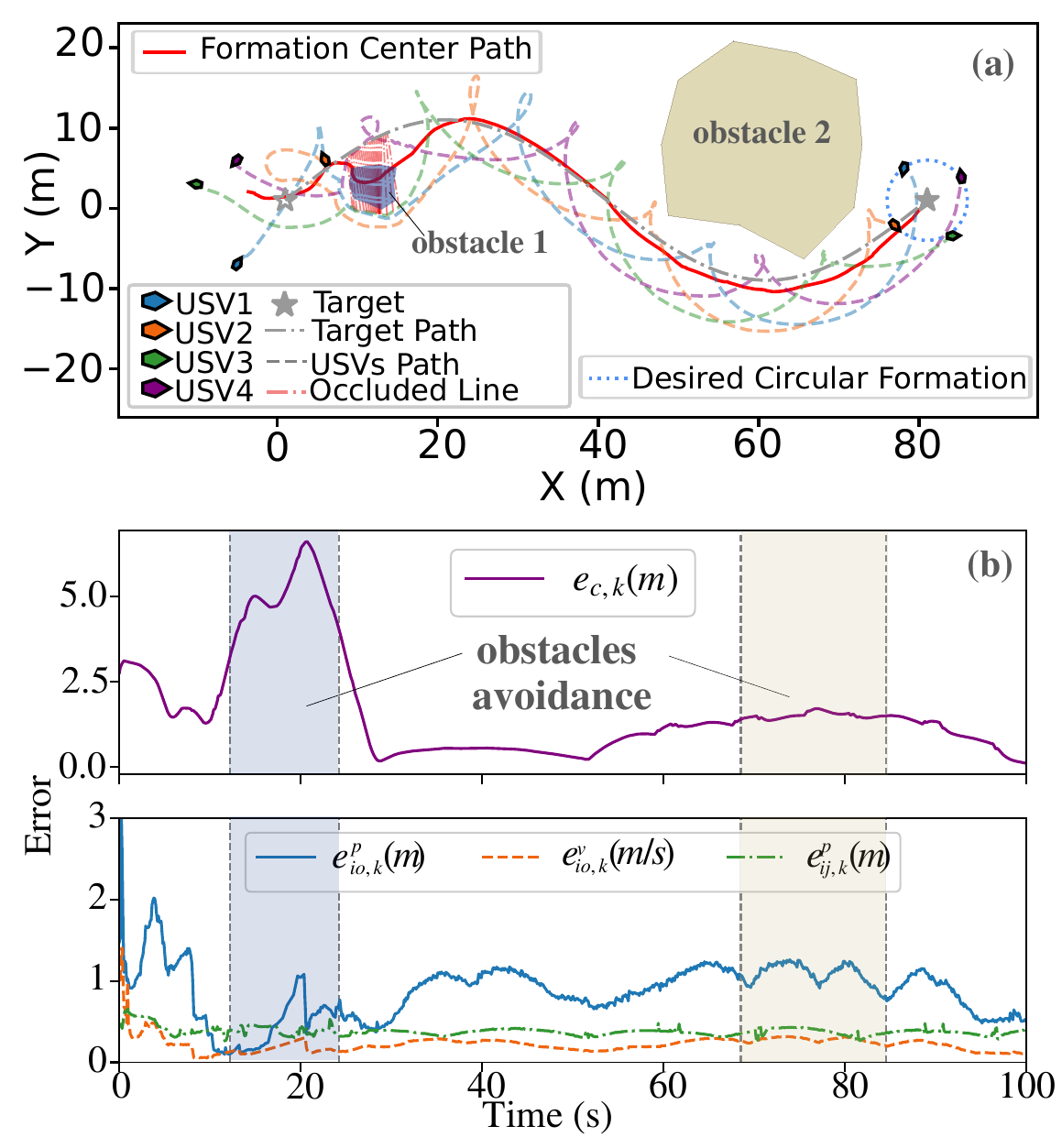}
    \caption{(a) Multi-USV circumnavigation a varying-velocity target with obstacle occlusion. Red dotted lines indicate line-of-sight occlusions between USVs and the target.(b)Error evolution for the obstacles occlusion scenario.}
    \label{fig:curveTrajandsta}
\end{figure}

\begin{figure}[htbp]
    \centering
    \includegraphics[width=0.47\textwidth]{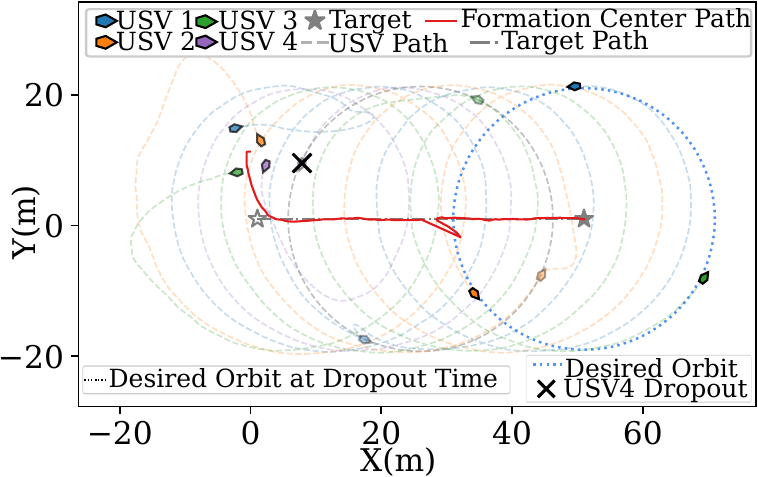}
    \caption{Trajectories of the USVs and the target in the USV dropout scenario.}
    \label{fig:DroppedTraj}
\end{figure}

\begin{figure}[htbp]
    \centering
    \includegraphics[width=0.47\textwidth]{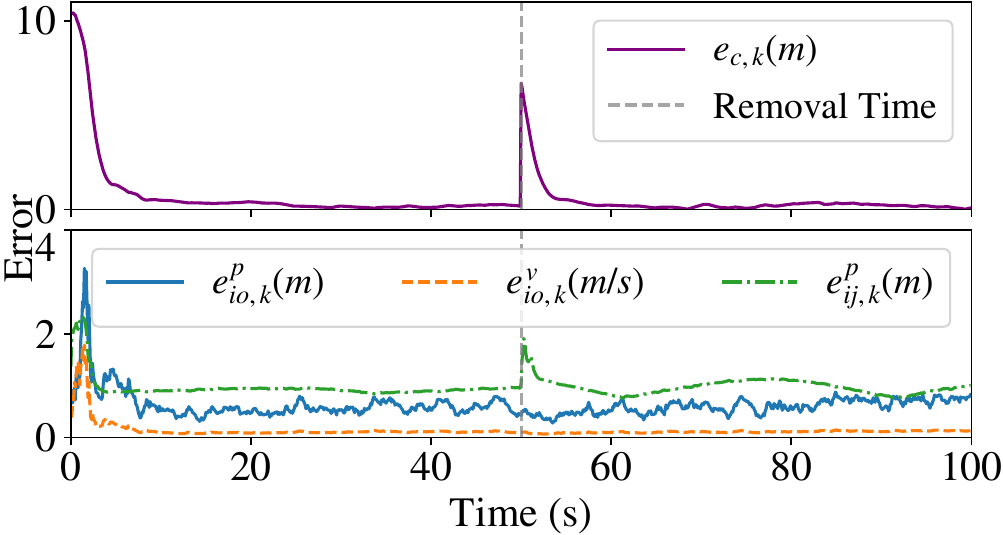}
    \caption{Error evolution for the USV dropout scenario.}
    \label{fig:DroppedSta}
\end{figure}

\textit{6) Performance under USV joining and dropout with communication interruptions:}
To validate the role of the coupled oscillator model within the overall framework, simulations involving USV joining and dropout are conducted under non-ideal communication conditions. Both sets of tests are performed with a 5\% communication packet loss rate. In each scenario, the desired formation radius is set to $20$m. Initially, four USVs are positioned at $(-3, 14)$m, $(2, 12)$m, $(-1, 8)$m, and $(3, 10)$m, respectively, with the target located at $(1, 1)$m.

\begin{figure}[htbp]
    \centering
    \includegraphics[width=0.47\textwidth]{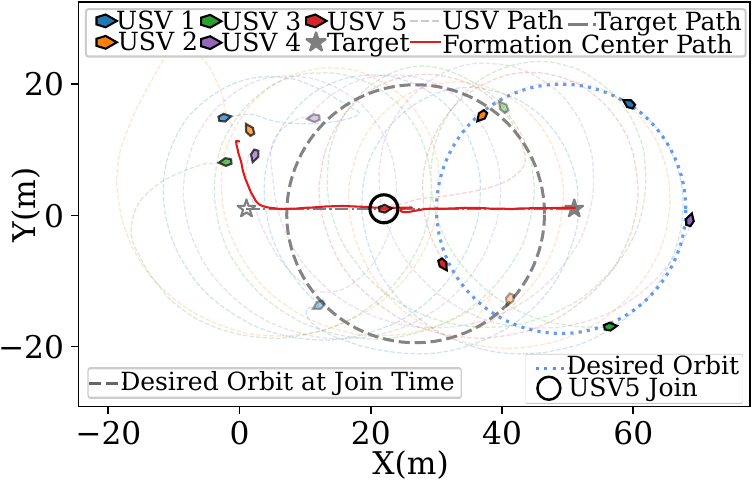}
    \caption{Trajectories of the USVs and the target in the USV joining scenario.}
    \label{fig:ADDTraj}
\end{figure}

\begin{figure}[htbp]
    \centering
    \includegraphics[width=0.47\textwidth]{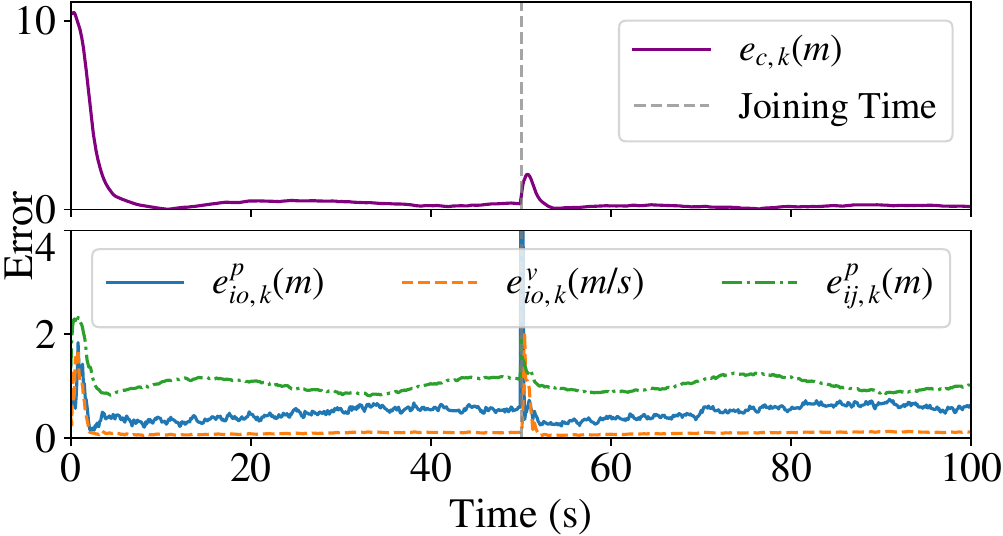}
    \caption{Error evolution for the USV joining scenario.}
    \label{fig:Addsta}
\end{figure}

Fig.~\ref{fig:DroppedTraj} illustrates the trajectories in the USV dropout scenario. Specifically, USV 4 is assumed to disconnect from the network at $t=50$s within a 100s simulation. Following the dropout, both the tracking error and estimation error of the system increase transiently, but quickly recover and converge within a short period (see Fig.~\ref{fig:DroppedSta}). As can be observed from the trajectories in Fig.~\ref{fig:DroppedTraj}, the remaining USVs autonomously readjust to a uniformly spaced configuration, thereby maintaining effective encirclement of the target.
Fig.~\ref{fig:ADDTraj} shows the trajectories in the USV join scenario. Here, USV 5 is introduced into the fleet at $t=50$s from an initial position of $(21, 1)$m. Similarly, upon the change in fleet size, the system error exhibits a transient increase before converging rapidly (see Fig.~\ref{fig:Addsta}). The proposed coupled oscillator model demonstrates remarkable resilience across both test cases, indicating that the system maintains stability without divergence despite sudden topological changes. Furthermore, the results show that the estimation error remains well within operational limits, even as the geometric constraints become more challenging due to the increased number of USVs.

\section{Conclusion}
This paper has presented a comprehensive framework for cooperative target circumnavigation using multiple USVs only with onboard sensing. A heterogeneous perception strategy was developed, integrating active inter-USV ranging and passive USV-target bearing measurements. To address these distinct estimation tasks, a Maximum Correntropy Kalman Filter was employed for inter-USV localization, while a Pseudo-Linear Kalman Filter was utilized for USV-target positioning. Furthermore, a coupled-oscillator-based controller is proposed to simultaneously guarantee target circumnavigation and ensure persistent excitation of the relative motions, thereby maintaining observability of the estimation system. The effectiveness and robustness of the proposed approach were validated through comprehensive simulations under various scenarios, including different target motion models and obstacle conditions. However, the USV dynamics model in this work is deliberately simplified and does not account for hydrodynamic effects. Consequently, to address these limitations, future research will investigate the integration of learning-based approaches to improve robustness against external disturbances and model uncertainties arising from hydrodynamics.

\section*{APPENDIX}
\subsection{Proof of Lemma~\ref{lemma:pe_vij}}
    From the combined analysis of controller \eqref{eq:controller} and vehicle dynamics \eqref{eq:dynamic_USV} \eqref{eq:USV_wel}, we obtain $\boldsymbol{v}_{ij,k} = \boldsymbol{\pi}_{ij,k} + \boldsymbol{v}_{ij,k}^*$, where 
    \begin{align*}
    \boldsymbol{\pi}_{ij,k} =& [ \text{sat}(U_{\boldsymbol{v}_{i}^f},\boldsymbol{v}_{i,k}^f) + \text{sat}(U_{\boldsymbol{v}_{o}},\hat{\boldsymbol{v}}_{o,k}^i) ] \\
    &- [ \text{sat}(U_{\boldsymbol{v}_{j}^f},\boldsymbol{v}_{j,k}^f) + \text{sat}(U_{\boldsymbol{v}_{o}},\hat{\boldsymbol{v}}_{o,k}^j) ] 
    \end{align*}
    Thus, $\|\boldsymbol{\pi}_{ij,k}\| \leq U_{\boldsymbol{v}_{i}^f} + U_{\boldsymbol{v}_{j}^f} + 2U_{\boldsymbol{v}_{o}}$. 
    
    To prove \eqref{eq:pe_vij_}, it suffices to show that for any unit vector $\boldsymbol{y} \in \mathbb{R}^2$, $ \lambda_{ij,1} \leq \boldsymbol{y}^T \boldsymbol S_{ij,l} \boldsymbol{y} \leq \lambda_{ij,2} $. Define
    \begin{equation*}
    \left\{\begin{array}{l}
\boldsymbol{V}_{ij,l}=\left[\boldsymbol{v}_{ij,l}^*,...,\boldsymbol{v}_{ij,l+N-1}^* \right] \\\boldsymbol{\Pi}_{ij,l}=\left[\boldsymbol{\pi}_{ij,l},...,\boldsymbol{\pi}_{ij,l+N-1} \right] 
    \end{array}
    \right.
    \end{equation*}
    Thus,  one has 
		\begin{equation}\label{eq:innerproduc}
			\boldsymbol{y}^T \boldsymbol S_{ij,l} \boldsymbol{y} = \left\lVert \boldsymbol{V}_{ij,l}^{T} \boldsymbol{y}  \right\rVert^{2} + 2\left\langle \boldsymbol{V}_{ij,l}^{T}\boldsymbol{y}, \boldsymbol{\Pi}_{ij,l}^{T}\boldsymbol{y} \right\rangle + \left\lVert \boldsymbol{\Pi}_{ij,l}^{T} \boldsymbol{y}\right\rVert^{2}
		\end{equation}	
    where $\langle \cdot, \cdot \rangle$ denotes the inner product. Note that 
    \begin{equation*}
        \lambda_{ij,1}^* \leq \lVert \boldsymbol{V}_{ij,l}^{T} \boldsymbol{y}  \rVert \leq \lambda_{ij,2}^*
    \end{equation*}
    and  
    \begin{equation*}
    \lVert \boldsymbol{\Pi}_{ij,l}^{T} \boldsymbol{y}  \rVert \leq \sqrt{N} \left(U_{\boldsymbol{v}_{i}^f} + U_{\boldsymbol{v}_{j}^f} + 2U_{\boldsymbol{v}_{o}}\right)
    \end{equation*} 
    
    By applying the Cauchy–Schwartz inequality,
    \begin{equation*}
    \begin{aligned}
         -\left\lVert \boldsymbol{V}_{ij,l}^{T} \boldsymbol{y}  \right\rVert \left\lVert \boldsymbol{\Pi}_{ij,l}^{T} \boldsymbol{y}  \right\rVert  &\leq \left\langle \boldsymbol{V}_{ij,l}^{T}\boldsymbol{y}, \boldsymbol{\Pi}_{ij,l}^{T}\boldsymbol{y} \right\rangle \\
         &\leq \left\lVert \boldsymbol{V}_{ij,l}^{T} \boldsymbol{y}  \right\rVert \left\lVert \boldsymbol{\Pi}_{ij,l}^{T} \boldsymbol{y}  \right\rVert 
    \end{aligned}
    \end{equation*}
It follows from the above inequalities that $\lambda_{ij,1} \leq \boldsymbol{y}^T \boldsymbol S_{ij,l} \boldsymbol{y} \leq \lambda_{ij,2}$. Thus, we have completed the proof of the conclusion given in \eqref{eq:pe_vij_} of Lemma \ref{lemma:pe_vij} .

\subsection{Proof of Lemma~\ref{lemma:pe_vio}}
    By analogy with the proof of Lemma \ref{lemma:pe_vij}, we obtain $\boldsymbol{v}_{io,k} = \boldsymbol{\pi}_{io,k} + \boldsymbol{v}_{io,k}^*$, where 
    \begin{equation*}
        \boldsymbol{\pi}_{io,k} = [ \text{sat}(U_{\boldsymbol{v}_{i}^f},\boldsymbol{v}_{i,k}^f) + \text{sat}(U_{\boldsymbol{v}_{o}},\hat{\boldsymbol{v}}_{o,k}^i) ] - \boldsymbol{v}_{o,k}
    \end{equation*}
    Consequently, $\|\boldsymbol{\pi}_{io,k}\| \leq U_{\boldsymbol{v}_{i}^f} + 2U_{\boldsymbol{v}_{o}}$. Then, following the same steps as in the proof of Lemma \ref{lemma:pe_vij}, we can readily establish the result stated in \eqref{eq:pe_vi_}.

\subsection{Proof of Theorem~\ref{theorem:MCKF_ob}} 
    By combining equations \eqref{eq:dynamic_USV}, \eqref{eq:USV_wel} and \eqref{eq:self_displacement} yields $$\boldsymbol{\delta}_{ij,k} = \boldsymbol{v}_{i,k} \Delta t - \boldsymbol{v}_{j,k} \Delta t  = \boldsymbol{v}_{ij,k} \Delta t$$
    Since $\boldsymbol{v}_{ij,k}$ is persistently exciting (Lemma \ref{lemma:pe_vij}), it follows from \eqref{eq:pe_vij_} that 
    \begin{equation*}
        \lambda_{ij,1}\Delta t \boldsymbol{I} \leqslant \boldsymbol \Phi_{ij,l}\triangleq \sum_{k=l}^{l+N-1}{\boldsymbol{\delta}_{ij,k} \boldsymbol{\delta}_{ij,k}^{T} } \leqslant \lambda_{ij,2} \Delta t \boldsymbol{I} 
    \end{equation*}
    which completes the proof.

\subsection{Proof of Theorem~\ref{theorem:PLKF_ob}} 
    We first show that the bearing vector $\boldsymbol{b}_{io,k}$ realtive to $\mathcal{V}_i$ and the target is persistently exciting. From the motion models in \eqref{eq:dynamic_USV} and \eqref{eq:dynamic_target}, we have $\boldsymbol{p}_{io,k+1} = \boldsymbol{p}_{io,k} + \Delta t \boldsymbol{v}_{io,k}$. Given that $\boldsymbol{v}_{io,k}$ is persistently exciting (Lemma \ref{lemma:pe_vio}), and $\|\boldsymbol{p}_{io,k}\| \leq U_{\boldsymbol{p}_{io}}$ (Assumption \ref{assump:maneuverability}), we can follow the same reasoning as in the proof of Lemma \ref{lemma:pe_vij} to obtain
     \begin{equation*}
     \begin{aligned}
    (\lambda_{io,1} -\Delta t \sqrt{N} U_{\boldsymbol{p}_{io}})^2 \boldsymbol{I} & \leq  \sum_{k=l}^{l+N-1}  {\boldsymbol{p}_{io,k} \boldsymbol{p}_{io,k}^{T} } \\ & \leq(\lambda_{io,2} +  \Delta t \sqrt{N} U_{\boldsymbol{p}_{io}})^2 \boldsymbol{I}
    \end{aligned}
    \end{equation*}
    By recalling the definition of $\boldsymbol{b}_{io,k}$ in \eqref{eq:target_bearing}, it can show that the $\boldsymbol{b}_{io,k}$ also satisfied the PE condition for $l \geq k_s$
     \begin{equation*}
         \sigma_{ij,1} \boldsymbol{I} \leqslant  \sum_{k=l}^{l+N-1}{\boldsymbol{b}_{io,k} \boldsymbol{b}_{io,k}^{T} }\leqslant \sigma_{ij,2} \boldsymbol{I}
     \end{equation*}
    where $ \sigma_{ij,1} = (\lambda_{io,1} - \Delta t \sqrt{N} U_{\boldsymbol{p}_{io}})^2 / U_{\boldsymbol{p}_{io}}^2 $ and $\sigma_{ij,2} = (\lambda_{io,2} + \Delta t \sqrt{N} U_{\boldsymbol{p}_{io}})^2/d_{s}^2$.

    We now compute $\boldsymbol \Phi_{io,l}$ to examine its structure. Using the fact that $ \boldsymbol{G}_{{\boldsymbol{b}}_{io},k}^T \boldsymbol{G}_{{\boldsymbol{b}}_{io},k} =\boldsymbol{G}_{{\boldsymbol{b}}_{io},k}$, we obtain \begin{equation*}
        \boldsymbol \Phi_{io,l} = \sum_{k=l}^{l+N-1} \left(\boldsymbol{I} - \boldsymbol{b}_{io,k}\boldsymbol{b}_{io,k}^T \right)
    \end{equation*}
    Let $\boldsymbol{q}_{io,k}$ denote the unit vector orthogonal to $\boldsymbol{b}_{io,k}$, so that $\boldsymbol{G}_{{\boldsymbol{b}}_{io},k} = \boldsymbol{q}_{io,k}\boldsymbol{q}_{io,k}^T$. Thus,
    $$\boldsymbol \Phi_{io,l} = \sum_{k=l}^{l+N-1} \boldsymbol{q}_{io,k} \boldsymbol{q}_{io,k} ^T $$  
    Note that $\boldsymbol{q}_{io,k}$ is also persistently exciting, satisfying 
    \begin{equation*}
        \sigma_{ij,1} \boldsymbol{I} \leqslant  \sum_{k=l}^{l+N-1}{\boldsymbol{q}_{io,k} \boldsymbol{q}_{io,k}^{T} }\leqslant \sigma_{ij,2} \boldsymbol{I}
    \end{equation*}
    Geometrically, this corresponds to a 90-degree rotation in the plane, which preserves the excitation properties of $\boldsymbol{b}_{io,k}$. It follows that the Gramian matrix $\boldsymbol{\Phi}_{io,l}$ is positive definite, implying observability of the relative positions $\boldsymbol{p}_{io,k}$. This completes the proof.

\bibliographystyle{IEEEtran} 

\bibliography{2025USV} %

\vfill

\end{document}